%% file: journal_kernel.tex
\documentclass[10pt,twocolumn]{IEEEtran}
\usepackage[dvipsnames]{xcolor}
\usepackage{amsmath,amssymb,amsthm,ascii,bm,bbm,capt-of,color,epsfig,hyphenat,mathrsfs,subfigure,stfloats,mathtools,marvosym}
\usepackage[english]{babel}
\usepackage[utf8]{inputenc}
\usepackage{graphicx}
\usepackage{amsmath}
\usepackage{hyperref, bm}
\usepackage{mathtools}
\usepackage{color}
\usepackage{amsmath}
\usepackage{lipsum}
\usepackage{multirow}
\usepackage{cancel}
\usepackage{pgfplots}
\usepackage{caption}
\usepackage{subfigure}
\usepackage[colaction]{multicol}
\usepackage{marginnote}
\usepackage[]{lineno}
\usepackage[normalem]{ulem}

\usepgfplotslibrary{groupplots}
\usetikzlibrary{plotmarks}
\pgfplotsset{minor tick num=1,
    legend style={font=\scriptsize},
    every axis label/.append style={font=\scriptsize},
    yticklabel style={/pgf/number format/fixed},
    every axis x label/.style={at={(0.5,0)},yshift=-15pt},
    every axis y label/.style={at={(0,0.5)},xshift=-30pt,rotate=90},
    tick label style = {font=\scriptsize},
    every axis plot/.append style={line width=0.6pt},
    axis background/.append style={fill=grey!2},
    every axis/.append style={max space between ticks=30},
    scaled y ticks=false
}
\newlength\figH
\newlength\figW
\usepackage{algpseudocode}
\usepackage{pifont}
\usepackage[ruled, noline]{algorithm2e}
\usepackage{upgreek}
\definecolor{grey}{gray}{0}
\usepackage[font=small,skip=4pt]{caption}

\title{\huge Matrix completion and extrapolation via kernel regression
 \thanks{This work was supported by the Ministerio de Economia y Competitividad of the Spanish Government and ERDF funds (TEC2016-75067-C4-2-R,TEC2015-515 69648-REDC),  Catalan Government funds (2017 SGR 578 AGAUR), and NSF grants (1500713, 1514056, 1711471 and 1509040).}}
\newcommand{\trace}[1]{\text{Tr}(#1)}
\newcommand{\ex}[1]{\mathbb{E}_{\bm e}\{ #1\}}
\newcommand{\norm}[1]{\left|\left|#1\right|\right|^2_2}

\newcommand{\nuclear}[1]{\left|\left|#1\right|\right|_*}
\newcommand{\frob}[1]{\left|\left|#1\right| \right|_\text{F}^2}
\newcommand{\frobs}[1]{\left|\left|#1\right| \right|_\text{F}}
\def\Y{\bm{Y}}

\def\y{\bm y}

\def\X{\bm X}
\def\x{\bm x}
\def\m{\bm m}
\def\mb{\bar{\bm m}}
\def\f{\bm f}
\def\g{\bm g}
\def\k{\bm k}

\def\e{\bm e}
\def\e{\bm{e}}

\def\M{\bm M}
\def\H{\bm H}

\def\v{\bm v}

\def\hz{\bm{\hat{v}}}
\def\A{\bm A}
\def\B{\bm B}

\def\I{\bm I}
\def\Q{\bm Q}

\def\W{\bm W}

\def\C{\bm C}
\def\T{\bm T}
\def\P{\bm P}
\def\S{\bm S}
\def\D{\bm D}
\def\U{\bm U}
\def\V{\bm V}
\def\L{\bm L}
\def\F{\bm F}
\def\E{\bm E}
\def\K{\bm K}

\def\tK{{\bm {\tilde K}}}
\def\w{\bm w}
\def\h{\bm h}

\def\gam{\bm{\gamma}}
\def\hgam{\hat{\bm{\gamma}}}
\def\hxi{\hat{\bm{\xi}}}
\def\bxi{\bm{\xi}}

\def\bphi{\bm \Phi}
\def\tbphi{\bm {\tilde{\Phi}}}

\def\Sig{\bm{\Sigma}}

\def\balp{\bm{\alpha}}

\def\kapx{\kappa_x}
\def\kapy{\kappa_y}
\def\kapz{\kappa_z}

\def\th{^\text{th}}

\DeclareMathOperator*{\argmin}{\arg\,\min}

\SetKwRepeat{Do}{do}{while}
\SetEndCharOfAlgoLine{}

\newtheorem{theorem}{Theorem}
\newtheorem{lemma}{Lemma}

\author{Pere~Gim\'enez-Febrer$^1$,
    Alba~Pag\`es-Zamora$^1$, and
    Georgios B. Giannakis$^2$\\
    $^1$SPCOM Group, Universitat Polit\`ecnica de Catalunya-Barcelona Tech, Spain\\
    $^2$Dept. of ECE and Digital Technology Center, University of Minnesota, USA
}
\begin{document}
\maketitle
\begin{abstract}
Matrix completion and extrapolation (MCEX) are dealt with here over reproducing kernel Hilbert spaces (RKHSs) in order to account for prior information present in the available data. Aiming at a fast and low-complexity solver, the task is formulated as one of kernel ridge regression. The resultant MCEX algorithm can also afford online implementation, while the class of kernel functions also encompasses several existing  approaches to MC with prior information. Numerical tests on synthetic and real datasets show that the novel approach is faster than widespread methods such as alternating least-squares (ALS) or stochastic gradient descent (SGD), and that the recovery error is reduced, especially when dealing with noisy data.
\end{abstract}
\section{Introduction}
With only a subset of its entries available, matrix completion (MC) amounts to recovering the unavailable entries by leveraging just the low-rank attribute of the matrix itself~\cite{candes}. The relevant task arises in applications as diverse as image restoration~\cite{ji2010robust}, sensor networks~\cite{yi2015partial}, and recommender systems~\cite{koren2009matrix}. To save power for instance, only a fraction of sensors may collect and transmit measurements to a fusion center, where the available spatio-temporal data can be organized in a matrix format, and the unavailable ones can be eventually interpolated via MC~\cite{yi2015partial}. Similarly,  collaborative filtering of ratings given by users to a small number of items are stored in a sparse matrix, and the objective is to predict their ratings for the rest of the items~\cite{koren2009matrix}.

Existing MC approaches rely on some form of rank minimization or low-rank matrix factorization. Specifically, ~\cite{candes} proves that when MC is formulated as the minimization of the nuclear norm subject to the constraint that the observed entries remain unchanged, exact recovery is possible under mild assumptions; see also~\cite{candes2010matrix} where reliable recovery from a few observations is established even in the presence of additive noise. Alternatively,~\cite{koren2009matrix} replaces the nuclear norm by two low-rank factor matrices that are identified in order to recover the complete matrix.

While the low-rank assumption can be sufficient for reliable recovery, prior information about the unknown matrix can be also accounted to improve the completion outcome. Forms of prior information can include sparsity~\cite{yi2015partial}, local smoothness~\cite{cheng2013stcdg}, and interdependencies encoded by graphs~\cite{kalofolias2014matrix,chen,rao2015collaborative,ma2011recommender}. These approaches exploit the available similarity information or prior knowledge of the bases spanning the column or row spaces of the unknown matrix. In this regard, reproducing kernel Hilbert spaces (RKHSs) constitute a powerful tool for leveraging available prior information thanks to the kernel functions, which measure the similarity between pairs of points in an input space. Prompted by this,~\cite{abernethy2006low,bazerque2013,zhou2012,stock2018comparative} postulate that columns of the factor matrices belong to a pair of RKHSs spanned by their respective kernels. In doing so, a given structure or similarity between rows or columns is effected on the recovered matrix. Upon choosing a suitable kernel function, \cite{yi2015partial} as well as \cite{cheng2013stcdg,kalofolias2014matrix,chen,rao2015collaborative,ma2011recommender} can be cast into the RKHS framework. In addition to improving MC performance, kernel-based approaches  also enable extrapolation of rows and columns, even when all their entries are missing - a task impossible by the standard MC approaches in e.g.~\cite{candes} and~\cite{koren2009matrix}.

One major hurdle in MC is the computational cost as the matrix size grows. In its formulation as a rank minimization task, MC can be solved via semidefinite programming~\cite{candes}, or proximal gradient minimization~\cite{ma2011fixed,cai2010singular,chen,gimenez}, which entails a singular value decomposition of the recovered matrix per iteration. Instead, algorithms with lower computational cost are available for the bi-convex formulation based on matrix factorization~\cite{koren2009matrix}. These commonly rely on iterative minimization schemes such as alternating least-squares (ALS)~\cite{hastie2015matrix,jain2013low} or stochastic gradient descent (SGD)~\cite{gemulla2011large,zhou2012}. With regard to kernel-based MC, the corresponding algorithms rely on alternating convex minimization and semidefinite programming~\cite{abernethy2006low}, block coordinate descent~\cite{bazerque2013}, and SGD~\cite{zhou2012}. However, algorithms based on alternating minimization only converge to the minimum after infinite iterations. In addition, existing kernel-based algorithms adopt a specific sampling pattern or do not effectively make use of the Representer Theorem for RKHSs that will turn out to be valuable in further reducing the complexity, especially when the number of observed entries is small.

The present contribution offers an RKHS-based approach to MCEX that also unifies and broadens the scope of MC approaches, while offering reduced complexity algorithms that scale well with the data size. Specifically, we develop a novel MC solver via kernel ridge regression as a convex alternative to the nonconvex factorization-based formulation that offers a closed-form solution. Through an explicit sampling matrix, the proposed method offers an encompassing sampling pattern, which further enables the derivation of upper bounds on the mean-square error. Moreover, an approximate solution to our MCEX regression formulation is developed that also enables online implementation using SGD. Finally, means of incorporating prior information through kernels is discussed in the RKHS framework.

The rest of the paper paper is organized as follows. Section II outlines the RKHS formulation and the kernel regression task. Section III unifies the existing methods for MC under the RKHS umbrella, while Section IV introduces our proposed Kronecker kernel MCEX (KKMCEX) approach. Section V develops our ridge regression MCEX (RRMCEX) algorithm, an accelerated version of KKMCEX, and its online variant. Section VI deals with the construction of kernel matrices. Finally, Section VII presents numerical tests, and Section VIII concludes the paper.

\noindent \textbf{Notation.} Boldface lower case fonts denote column vectors, and boldface uppercase fonts denote matrices. The $(i,j)$th entry of matrix $\A$ is $\A_{i,j}$, and the $i^{\text{th}}$ entry of vector $\bm a$ is $\bm a_i$. Superscripts $^T$ and $^\dagger$ denote transpose and pseudoinverse, respectively; while hat $\,\hat{}\,$ is used for estimates. Matrix $\F\in\mathcal{H}$ means that its columns belong to a vector space $\mathcal{H}$. The symbols $\I$ and $\bm 1$ stand for the identity matrix and the all-ones vector of appropriate size, specified by the context. The trace operator is $\text{Tr}(\cdot)$, the function eig($\A$) returns the diagonal eigenvalue matrix of $\A$ ordered in ascending order, and $\lambda_k(\A)$ denotes the $k^{\text{th}}$ eigenvalue of $\A$ with $\lambda_k(\A)\leq\lambda_{k+1}(\A)$.
\section{Preliminaries}
\label{sec:background}
Consider a set of $N$ input-measurement pairs $\{(x_i,m_i)\}^N_{i=1}$ in $\mathcal{X}\times\mathbb{R}$, where $\mathcal{X}:=\{x_1,\ldots,x_N\}$ is the input space, $\mathbb{R}$ denotes the set of real numbers, and measurements obey the model
\begin{equation}\label{eq:sigmodel}
  m_i = f(x_i) + e_i
\end{equation}
where $f:\mathcal{X}\rightarrow\mathbb{R}$ is an unknown function and $e_i\in\mathbb{R}$ is noise. We assume this function belongs to an RKHS 
\begin{equation}\label{eq:hilb}
  \mathcal{H}_x:=\{f:f(x_i) = \sum^N_{j=1} \alpha_j \kapx(x_i,x_j), \:\:\: \alpha_j \in \mathbb{R}\}
\end{equation}
where $\kapx:\mathcal{X}\times\mathcal{X}\rightarrow\mathbb{R}$ is the kernel function that spans $\mathcal{H}_x$, and $\{\alpha_i\}^N_{i=1}$ are weight coefficients. An RKHS is a complete linear space endowed with an inner product that satisfies the reproducing property~\cite{shawe}. If $\langle f,f'\rangle_{\mathcal{H}_x}$ denotes the inner product in $\mathcal{H}_x$ between functions $f$ and $f'$, the reproducing property states that $f(x_i) = \langle f,\kapx(\cdot,x_i)\rangle_{\mathcal{H}_x}$; that is, $f$ in $\mathcal{H}_x$ can be evaluated at $x_i$ by taking the inner product between $f$ and $\kapx(\cdot,x_i)$. With $\{\alpha_i\}^N_{i=1}$ and $\{\alpha'_i\}^N_{i=1}$ denoting the coefficients of $f$ and $f'$ in (\ref{eq:hilb}) respectively, we have $\langle f,f'\rangle_{\mathcal{H}_x} := \sum_{i=1}^N\sum_{j=1}^N \alpha_i\alpha_j' \kapx(x_i,x_j)$; that is,
\begin{align}
  \langle f,f'\rangle_{\mathcal{H}_x} = \balp^T\K_x\balp'\label{eq:innerprod}
\end{align}
where $\balp := [\alpha_1,\ldots,\alpha_N]^T$, $\balp' := [\alpha_1',\ldots,\alpha_N']^T$ and $(\K_x)_{i,j}:=\kapx(x_i,x_j)$. In order for $\langle \cdot,\cdot \rangle_{\mathcal{H}_x}$ in (\ref{eq:innerprod}) to be an inner product, $\kapx$ must be symmetric and semipositive definite, meaning $\langle f,f\rangle_{\mathcal{H}_x} \geq 0\: \forall f\in\mathcal{H}_x$. As a consequence, $\K_x$ will be symmetric positive semidefinite since $\balp^T\K_x\balp \geq 0 \:\: \forall\, \balp \in \mathbb{R}^N$. 

While $\kapx$ is usually interpreted as a measure of similarity between two elements in $\mathcal{X}$, it can also be seen as the inner product of corresponding two elements in feature space $\mathcal{F}$ to which $\mathcal{X}$ can be mapped using function $\phi_x:\mathcal{X}\rightarrow\mathcal{F}$. Formally, we write 
\begin{equation}\label{eq:kerdot}
    \kapx(x_i,x_j)=\langle\phi_x(x_i),\phi_x(x_j)\rangle_{\mathcal{F}}.
\end{equation}
Function $\phi_x$ is referred to as feature map, and its choice depends on the application. For an input space of text files, for example, the files could be mapped to a feature vector that tracks the number of words, lines, and blank spaces in the file. Since $\phi_x$ can potentially have infinite dimension, evaluating the kernel using~\eqref{eq:kerdot} might be prohibitively expensive. This motivates specifying the kernel through a similarity function in $\mathcal{X}$, which bypasses the explicit computation of the inner product in $\mathcal{F}$. Typical examples include the Gaussian kernel $\kapx(x_i,x_j) = \text{exp}\{-\norm{x_i-x_j}/(2\eta)\}$ with $\eta$ being a free parameter, and the polynomial kernel~\cite{friedman2001elements}. In certain cases however, it is difficult to obtain the kernel similarity function on the input space. Such cases include metric input spaces with misses (as in MC), and non-metric spaces. 
The alternative to both is deriving the kernel from prior information. For instance, if we have a graph connecting the points in $\mathcal{X}$, a kernel can be obtained from the graph Laplacian~\cite{romero}. 

Having introduced the basics of RKHS, we proceed with the kernel regression task, where given $\{m_i\}^N_{i=1}$ we seek to obtain 
\begin{equation}\label{eq:rrf}
    \hat{f} = \argmin_{f\in\mathcal{H}_x} {1\over N}\sum_{i=1}^N l(m_i,f(x_i)) +  \mu' ||f||^2_{\mathcal{H}_x}
\end{equation}
with $l(\cdot)$ denoting the loss,  $\mu' \in \mathbb{R}^+$ the regularization parameter, and $||f||_{\mathcal{H}_x}:=\langle f,f\rangle_{\mathcal{H}_x}$ the norm induced by the inner product  in (\ref{eq:innerprod}). We will henceforth focus on the square loss $l(m_i,f(x_i)) := (m_i - f(x_i))^2$. 
Using $\K_x$, consider without loss of generality expressing the vector $\f:=[f(x_1),\ldots,f(x_N)]^T$ as $\f=\K_x\balp$, where $\balp:=[\alpha_1,\ldots,\alpha_N]^T$. Using the latter in the square loss,~(\ref{eq:rrf}) boils down to a kernel ridge regression (KRR) problem that can be solved by estimating $\balp$ as
\begin{equation}\label{eq:rralp}
  \hat{\balp} = \argmin_{\balp\in\mathbb{R}^N} \norm{\m - \K_x \balp} +\mu \mspace{2mu} \balp^T \K_x \balp
\end{equation}
where $\m := [m_1,\ldots,m_N]^T$ and $\mu=N\mu'$. The weights can be found in closed form as 
\begin{equation}\label{eq:rralpest}
\hat{\balp} = (\K_x+\mu\I)^{-1}\m
\end{equation}
and the estimate of the sought function is obtained as $\hat{\f}=\K_x\hat{\balp}$. 

\section{Kernel-based MCEX}
\label{sec:mc}

Matrix completion considers $\F\in\mathbb{R}^{N\times L}$ of rank $r$ observed through a $N\times L$ matrix of noisy observations
\begin{equation}\label{eq:matm}
    \M = P_{\Omega}(\F+\bm \E)
\end{equation}
where $\Omega\subseteq\{1,\ldots,N\}\times \{1,\ldots,L\}$ is the sampling set of cardinality $S=|\Omega|$ containing the indices of the observed entries; $P_{\Omega}(\cdot)$ is a projection operator that sets to zero the entries with index $(i,j)\notin \Omega$ and leaves the rest unchanged; and, $\bm\E\in\mathbb{R}^{N\times L}$ is a noise matrix. According to~\cite{candes2010matrix}, one can recover $\bm F$ from $\M$ with an error proportional to the magnitude of $\frob{\E}$ by solving the following convex optimization problem:
\begin{align}\label{eq:nnormc}
    &\min_{\F\in{\mathbb{R}^{N\times L}}}
    & \mspace{-50mu}& \text{rank}(\F)  \nonumber \\
    & \text{subject to}
    & \mspace{-50mu}& \frob{P_{\Omega}(\F - \M)} \leq \delta
\end{align}
where  $\frobs{\cdot}$ is the Frobenius norm, and we assume $\frob{P_\Omega(\E)} \leq \delta$ for some $\delta > 0$.
Since solving (\ref{eq:nnormc}) is NP-hard, the nuclear norm $\nuclear{\F} := \text{Tr}(\sqrt{\F^T\F})$ can be used to replace the rank to obtain the convex problem~\cite{ma2011fixed,chen}
\begin{equation}\label{eq:nuclmc}
  \min_{\F\in{\mathbb{R}^{N\times L}}} \frob{P_\Omega(\M-\F)} + \mu\nuclear{\F}.
\end{equation}
Because $\F$ is low rank, it is always possible to factorize it as $\F=\W\H^T$, where $\W \in \mathbb{R}^{N\times p}$ and $\H \in \mathbb{R}^{L\times p}$ are the latent factor matrices with $p \geq r$. This factorization allows expressing the nuclear norm as~\cite{srebro2005maximum} $\nuclear{\F} = \min_{\F=\W\H^T} {1\over 2}\left(\frob{\W} + \frob{\H}\right)$
which allows reformulating (\ref{eq:nuclmc}) as 
\begin{equation}\label{eq:mclag}
    \{\hat{\W}\!,\!\hat{\H}\} \!=\! \argmin_{\substack{\W\in\mathbb{R}^{N\times p}\\\H\in\mathbb{R}^{L\times p}}} \frob{P_\Omega(\!\M\!-\!\W\H^T)} \!+ \mu\!\left(\frob{\W} \!+\! \frob{\H}\right)
\end{equation}
and yields $\hat{\F}=\hat{\W}\hat{\H}^T$. While the solutions to (\ref{eq:nuclmc}) and (\ref{eq:mclag}) are equivalent when the rank of the matrix minimizing~\eqref{eq:nuclmc} is smaller than $p$\cite{hastie2015matrix}, solving (\ref{eq:nuclmc}) can be costlier since it involves the computation of the singular values of the matrix. On the other hand, since (\ref{eq:mclag}) is bi-convex it can be solved by alternately optimizing $\W$ and $\H$, e.g. via ALS~\cite{jain2013low} or SGD iterations~\cite{gemulla2011large}. Moreover, leveraging the structure of (\ref{eq:mclag}), it is also possible to optimize one row from each factor matrix at a time instead of updating the full factor matrices, which enables faster and also online and distributed implementations~\cite{teflioudi2012}.


Aiming at a kernel-based MCEX, we model the columns and rows of $\F$ as functions that belong to two different RKHSs. To this end, consider the input spaces $\mathcal{X}:=\{x_1,\ldots,x_N\}$ and $\mathcal{Y}:=\{y_1,\ldots,y_L\}$ for the column and row functions, respectively. In the user-movie ratings paradigm, $\mathcal{X}$ could be the set of users, and $\mathcal{Y}$ the set of movies. Then $\F:=[\f_1,\ldots,\f_L]$ is formed with columns $\f_l := [f_l(x_1),\ldots,f_l(x_N)]^T$
 with $f_l:\mathcal{X}\rightarrow\mathbb{R}$. Likewise, we rewrite $\F := [\g_1,\ldots,\g_N]^T$, with rows $\g_n^T:= [g_n(y_1),\ldots,g_n(y_L)]$
and $g_n:\mathcal{Y}\rightarrow\mathbb{R}$ .  We further assume that $f_l\in\mathcal{H}_x\: \forall l=1,\ldots,L$ and $g_n\in\mathcal{H}_y \: \forall n=1,\ldots,N$, where 
\begin{equation}\label{eq:hilbrow}
  \mathcal{H}_x:=\{f:f(x_i) = \sum^N_{j=1} \alpha_j \kapx(x_i,x_j), \:\:\: \alpha_j \in \mathbb{R}\}
\end{equation}
\begin{equation}\label{eq:hilbcol}
  \mathcal{H}_y:=\{g:g(y_i) = \sum^L_{j=1} \beta_j \kapy(y_i,y_j), \:\:\: \beta_j \in \mathbb{R}\}
\end{equation}
and $\kapx:\mathcal{X}\times\mathcal{X}\rightarrow \mathbb{R}$ and $\kapy:\mathcal{Y}\times\mathcal{Y}\rightarrow \mathbb{R}$ are the kernels forming $\K_x\in\mathbb{R}^{N\times N}$ and $\K_y\in\mathbb{R}^{L\times L}$, respectively. 

Since $\W$ and $\H$ span the column and row spaces of $\F$, their columns belong to $\mathcal{H}_x$ and $ \mathcal{H}_y$ as well. Thus, the $m^\text{th}$ column of $\W$ is
\begin{equation}\label{eq:vecw}
    \w_m:=[w_m(x_1),\ldots,w_m(x_N)]^T
\end{equation}
where $w_m:\mathcal{X}\rightarrow\mathbb{R}$ and $w_m\in\mathcal{H}_x\:\forall m=1,\ldots,p$, and the $m^\text{th}$ column of $\H$ is
\begin{equation}\label{eq:vech}
\h_m:=[h_m(y_1),\ldots,h_m(y_L)]^T
\end{equation}
where $h_m:\mathcal{Y}\rightarrow\mathbb{R}$ and $h_m\in\mathcal{H}_y\:\forall m=1,\ldots,p$. Hence, instead of simply promoting a small Frobenius norm for the factor matrices as in (\ref{eq:mclag}), we can also promote smoothness on their respective RKHS. The kernel-based formulation in~\cite{bazerque2013} estimates the factor matrices by solving
\begin{align}
    \{\hat{\W},\hat{\H}\}=\argmin_{\substack{\W\in\mathcal{H}_x\\\H\in\mathcal{H}_y}} \ &\frob{P_\Omega(\M-\W\H^T)}  \label{eq:factker}\\& \mspace{-50mu}+ \mu\trace{\W^T\K_x^{-1}\W} + \mu\trace{\H^T\K_y^{-1}\H}.\nonumber
\end{align}
Note that (\ref{eq:factker}) is equivalent to (\ref{eq:mclag}) for $\K_x=\I$ and $\K_y=\I$. Since the constraints $\W\in\mathcal{H}_x$ and $\H\in\mathcal{H}_y$  can be challenging to account for when solving (\ref{eq:factker}),  we can instead find the coefficients that generate $\W$ and $\H$ in their respective RKHSs in order to satisfy such constraints. Thus, if we expand $\W=\K_x\B$ and $\H=\K_y\C$, where $\B\in\mathbb{R}^{N\times p}$ and $\C\in\mathbb{R}^{L\times p}$ are coefficient matrices, (\ref{eq:factker}) becomes
\begin{align}\label{eq:factkerw}
    \{\hat{\B},\hat{\C}\}=\argmin_{\substack{\B\in\mathbb{R}^{N\times p}\\\C\in\mathbb{R}^{L\times p}}} & \frob{P_\Omega(\M-\K_x\B\C^T\K_y)}  \\
    &+ \mu\trace{\B^T\K_x\B} + \mu\trace{\C^T\K_y\C}.\nonumber
\end{align}
Nevertheless, with nonsingular kernel matrices, $\B$ and $\C$ can be found by solving \eqref{eq:factker} and substituting $\hat{\B}=\K_x^{-1}\hat{\W}$ and $\hat{\C}=\K_y^{-1}\hat{\H}$~\cite{bazerque2013}.


Alternating minimization schemes that solve the bi-linear MC formulation \eqref{eq:mclag} tends to the solution to the convex problem \eqref{eq:nuclmc} in the limit~\cite{jain2013low}, thus convergence to the global optimum is not guaranteed unless the number of iterations is infinite. Since algorithms for kernel-based MC~\cite{bazerque2013} solving \eqref{eq:factker} rely on such alternating minimization schemes, they lack convergence guarantees given finite iterations as well. In addition to that, their computational cost scales with the size of $\F$. On the other hand, online implementations have a lower cost~\cite{zhou2012}, but only guarantee convergence to a stationary point~\cite{mardani2015}. In the ensuing section we develop a convex kernel-based reformulation of MCEX that enables a closed-form solver which purely exploits the extrapolation facilitated by the kernels. By casting aside the low-rank constraints, the computational complexity of our solver scales only with the number of observations while, according to our numerical tests, providing better performance. Moreover, we derive an online implementation that can be seamlessly  extended to distributed operation.

\section{Kronecker kernel MCEX}
\label{sec:kmr}

In the previous section, we viewed the columns and rows of $\F$ as functions evaluated at the points of the input spaces $\mathcal{X}$ and $\mathcal{Y}$ in order to unify the state-of-the-art on MC using RKHSs. Instead, we now postulate entries of $\F$ as the output of a function lying on an RKHS evaluated at a tuple $(x_i,y_i) \in \mathcal{X} \times \mathcal{Y}$. Given the spaces $\mathcal{X}$ and $\mathcal{Y}$, consider the space $\mathcal{Z}:=\mathcal{X}\times\mathcal{Y}$ with cardinality $|\mathcal{Z}|=NL$ along with the two-dimensional function $v:\mathcal{Z}\rightarrow\mathbb{R}$ as $v(x_i,y_j)=f_j(x_i)$, which belongs to the RKHS
\begin{equation}\label{eq:hilbz}
  \mathcal{H}_z\!:=\!\{v\!:\!v(x_i,\!y_j) \!=\! \sum^N_{n=1}\!\sum^L_{l=1} \!\gamma_{n,l} \kapz((x_i,\!y_j),\!(x_n,\!y_l)), \: \:\; \gamma_{n,l} \!\in \!\mathbb{R}\}
\end{equation}
with $\kapz:\mathcal{Z}\times\mathcal{Z}\rightarrow \mathbb{R}$. While one may choose any kernel to span $\mathcal{H}_z$, we will construct one adhering to our bilinear factorization $\F=\W\H^T$ whose $(i,j)^\text{th}$ entry yields
\begin{equation}\label{eq:zhw}
\F_{ij}=v(x_i,y_j) = \sum^p_{m=1}w_m(x_i)h_m(y_j)
\end{equation}
with $w_m$ and $h_m$ functions capturing $m^{\text{th}}$ column vector of $\W$ and $\H$ as in (\ref{eq:vecw}) and (\ref{eq:vech}). Since $ w\in\mathcal{H}_x$ and $ h\in\mathcal{H}_y$, we can write  $w_m(x_i) = \sum^N_{n=1} b_{n,m} \kapx(x_i,x_n)$ and $h_m(y_j) = \sum^L_{l=1} c_{l,m} \kapy(y_j,y_l)$, where $b_{n,m}$ and $c_{l,m}$ are the entries at $(n,m)$ and $(l,m)$ of the factor matrices $\B$ and $\C$ from~\eqref{eq:factkerw}, respectively. Therefore,~(\ref{eq:zhw}) can be rewritten as
\begin{align}
v(x_i,y_j) &= \sum^p_{m=1} \sum^N_{n=1} b_{n,m}\kapx(x_i,x_n) \sum^L_{l=1} c_{l,m}\kapy(y_j,y_l) \nonumber \\
&= \sum^N_{n=1}\sum^L_{l=1} \left(\sum^p_{m=1} b_{n,m} c_{l,m} \right) \kapx(x_i,x_n)\kapy(y_j,y_l) \nonumber \\
&=  \sum^N_{n=1}\sum^L_{l=1} \gamma_{n,l} \kapz((x_i,y_j),(x_n,y_l))\label{eq:zfunc}
\end{align}
where $\gamma_{n,l}=\sum^p_{m=1} b_{m,n} c_{m,l}$, and $\kapz((x_i,y_j),(x_n,y_l)) =\kapx(x_i,x_n)\kapy(y_j,y_l)$
since a product of kernels is itself a kernel~\cite{friedman2001elements}.
Using the latter,~\eqref{eq:zfunc} can be written compactly as
\begin{equation}\label{eq:zkermat}
    v(x_i,y_j) = \k_{i,j}^T\gam
\end{equation}
where $\gam:=[\gamma_{1,1},\gamma_{2,1},\ldots,\gamma_{N,1},\gamma_{1,2},\gamma_{2,2},\ldots,\gamma_{N,L}]^T$, and correspondingly,
\begin{align} \label{eq:kykz}
    \k_{i,j}=&[\kapx(x_i,x_1)\kapy(y_j,y_1), \ldots, \kapx(x_i,x_N)\kapy(y_j,y_1),\nonumber \\ &\kapx(x_i,x_1)\kapy(y_j,y_2), \ldots, \kapx(x_i,x_N)\kapy(y_j,y_L)]^T \nonumber \\
    =& (\K_y)_{:,j}\otimes(\K_x)_{:,i}
\end{align}
where a subscript $(:,j)$ denotes the $j^{\text{th}}$ column of a matrix, and we have used that $\K_x$ and $\K_y$ are symmetric matrices. In accordance with \eqref{eq:kykz}, the kernel matrix of $\mathcal{H}_z$ in~\eqref{eq:hilbz} is
\begin{equation}\label{eq:Kz}
\K_z=\K_y\otimes\K_x.
\end{equation}
Clearly, $\k_{i,j}$ in~\eqref{eq:kykz} can also be expressed as $\k_{i,j}=(\K_z)_{:,(j-1)N+i}$. This together with~\eqref{eq:zkermat} implies that 
\begin{align}
  \v =[&v(x_1,y_1),v(x_2,y_1),\ldots,v(x_N,y_1),v(x_1,y_2),\\&v(x_2,y_2),\ldots,v(x_N,y_N)]^T 
\end{align}
can be expressed in matrix-vector form as
\begin{equation}\label{eq:zmod}
\v = \K_z\gam
\end{equation}
or, equivalently, $\v=\text{vec}(\F)$. Note that entries of the kernel matrix are $(\K_z)_{i',j'} = \kapx(x_i,x_n)\kapy(y_j,y_l)$, where $n=j'\:\text{mod}\:N, i=i'\:\text{mod}\:N, l=\lceil {j'\over N}\rceil, \text{ and } j=\lceil {i'\over N}\rceil$.

Since the eigenvalues of $\K_z$ are the product of eigenvalues of $\K_y$ and $\K_x$, it follows that $\K_z$ is positive semidefinite and thus a valid kernel matrix. With the definition of the function $v$ and its vector form we have transformed the matrix of functions specifying $\F$ into a function that lies on the RKHS $\mathcal{H}_z$. Hence, we are ready to formulate MCEX as a kernel regression task for recovering $\v$ from the observed entries of $\m=\text{vec}(\M)$.

 Given  $ \{((x_i,y_j),m_{i,j})\}_{(i,j)\in\Omega}$ in $\mathcal{Z}\times\mathbb{R}$, our goal is to recover the function $v$ as
 \begin{equation}\label{eq:lossr1}
  \hat{v}=\argmin_{v\in\mathcal{H}_z} \sum_{(i,j)\in\Omega} (m_{i,j}-v(x_i,y_j))^2 + \mu ||v||^2_{\mathcal{H}_\mathcal{Z}}
\end{equation}
where $||v||^2_{\mathcal{H}_\mathcal{Z}}:=\gam^T\K_z\gam$.
 Define next $\e := \text{vec}(\E)$ and $\mb=\S\m$, where $\S$ is an $S\times NL$ binary sampling matrix also used to specify the sampled noise vector $\bar{\e}=\S\e$. With these definitions and (\ref{eq:zmod}), the model in (\ref{eq:matm}) becomes
\begin{equation}\label{eq:obsvec}
  \mb = \S\v + \S\e = \S\K_z\gam + \bar{\e}
\end{equation}
which can be solved to obtain
\begin{equation}\label{eq:mcker}
  \hgam = \argmin_{\gam\in\mathbb{R}^{NL}} \norm{\mb - \S\K_z\gam} + \mu\gam^T\K_z\gam
\end{equation}
in closed form
\begin{equation}\label{eq:krr}
  \hgam = (\S^T\S\K_z+\mu\I)^{-1}\S^T\mb.
\end{equation}
Since the size of $\K_z$ is $NL\times NL$, the inversion in (\ref{eq:krr}) can be very computationally intensive. To alleviate this, we will leverage the Representer Theorem (see~\cite{scholkopf2001generalized} for a formal proof), which allows us to reduce the number of degrees of freedom of the regression problem. In our setup, this theorem is as follows. 
\begin{theorem}
\noindent \textbf{Representer Theorem}. Given the set of input-observations pairs $\{(x_i,y_j),m_{i,j})\}_{(i,j)\in \Omega}$ in $\mathcal{Z}\times \mathbb{R}$ and the function $v$ as in (\ref{eq:zfunc}), the solution to 
\begin{equation}\label{eq:lossr}
  \argmin_{v\in\mathcal{H}_z} \sum_{(i,j)\in\Omega} (m_{i,j}-v(x_i,y_j))^2 + \mu ||v||^2_{\mathcal{H}_\mathcal{Z}}
\end{equation}
is an estimate $\hat{v}$ that satisfies
\begin{equation}\label{eq:zhat}
  \hat{v} = \sum_{(n,l)\in\Omega} {\tau}_{n,l} k_{z}((\cdot,\cdot),(x_n,y_l))
\end{equation}
for some coefficients $\tau_{n,l}\in\mathbb{R},$ $\forall(n,l)\in\Omega$.
\end{theorem}
Theorem 1 asserts that $\hat{\gam}$ in (\ref{eq:mcker}) satisfies 
$\hat{\gamma}_{n,l}=0 \; \forall \; (n,l) \notin \Omega$. Therefore, we only need to optimize $\{\gamma_{n,l} : (n,l)\in\Omega\}$ which correspond to the observed pairs. In fact, for our vector-based formulation, the Representer Theorem boils down to applying on~\eqref{eq:krr} the matrix inversion lemma (MIL), which asserts the following. 
    \begin{lemma}\label{lem:MIL} \textbf{MIL}~\cite{henderson1981}. Given matrices $\A,\U$ and $\V$ of conformal dimensions, with $\A$ invertible, it holds that
\begin{equation}\label{eq:mil}
   (\U\V+\A)^{-1}\U = \A^{-1}\U(\V\A^{-1}\U+\I).
\end{equation}
\end{lemma}
With \eqref{eq:krr} $\A=\mu\I, \U=\S^T$ and $\V=\K_z\S^T$, application of \eqref{eq:mil} to ~\eqref{eq:krr} yields
\begin{equation}\label{eq:gamest}
  \hat{\gam} = \S^T(\S\K_z\S^T+\mu\I)^{-1}\mb.
\end{equation}
Subsequently, we reconstruct $\v$ as
\begin{equation}\label{eq:zKKMCEX}
  \hat{\v}_K = \K_z\S^T(\S\K_z\S^T+\mu\I)^{-1}\mb
\end{equation}
and we will henceforth refer to as the \textit{Kronecker kernel MCEX} (KKMCEX) estimate of $\v$. Regarding the computational cost incurred by \eqref{eq:gamest}, inversion costs $\mathcal{O}(S^3)$, since the size of the matrix to be inverted is reduced from $NL$ to $S$. Clearly, there is no need to compute $\K_z=\K_y\otimes\K_x$. As $\S$ has binary entries, $\S\K_z\S^T$ is just a selection of $S^2$ entries in $\K_z$; and, given that $\kapz((x_i,y_j),(x_n,y_l)) =\kapx(x_i,x_n)\kapy(y_j,y_l)$, it is obtained at cost $\mathcal{O}(S^2)$. Overall, the cost incurred by \eqref{eq:gamest} is $\mathcal{O}(S^3)$. Compared to the MC approach in \eqref{eq:factker}, the KKMCEX method is easier to implement since it only involves a matrix inversion. Moreover, since it admits a closed-form solution, it facilitates deriving bounds on the estimation error of $\hat{\v}_K$.

\noindent \textbf{Remark 1}. Matrices built via the Kronecker product have been used in regression for different purposes. Related to MC, \cite{rao2015collaborative} leverages Kronecker product structures to efficiently solve the Sylvester equations that arise in alternating minimization iterations to find $\{\hat{\W},\hat{\H}\}$ in \eqref{eq:factker}. On the other hand, \cite{pahikkala2014two,stock2018comparative} propose a Kronecker kernel ridge regression method that can be used to extrapolate missing entries in a matrix. However, the methods in~\cite{pahikkala2014two,stock2018comparative} assume a complete training set and Kronecker structure for the regression matrix; this implies that the observed entries in $\bm{M}$ can be permuted to form a full submatrix. In our formulation, we introduce $\S$ which encompasses any sampling pattern in $\Omega$. Thus, the properties of the Kronecker product used in~\cite{rao2015collaborative,pahikkala2014two,stock2018comparative} cannot be applied to solve~$\eqref{eq:zKKMCEX}$ since $\S\K_z\S^T$ is not necessarily the Kronecker product of two smaller matrices.

\noindent \textbf{Remark 2}. The KKMCEX solution in~\eqref{eq:zKKMCEX}, differs from that obtained as the solution of~(\ref{eq:factkerw}). On the one hand, the loss in~\eqref{eq:mcker} can be derived from the factorization-based one by using the Kronecker product kernel $\K_y\otimes\K_x$ and $\gam=\text{vec}(\B\C^T)$ to arrive at
\begin{align}
& \frob{P_\Omega(\M - \K_x\B\C^T\K_y)} \nonumber\\&  \:\: = \norm{\mb - \S(\K_y\otimes\K_x)\text{vec}(\B\C^T)}.
\end{align}
One difference between the two loss functions is that (\ref{eq:mcker}) does not explicitly limit the rank of the recovered matrix $\hat{\F}=\text{unvec}({\hat{\v}_R})$ since it has $NL$ degrees of freedom through $\hat{\gam}$, while in (\ref{eq:factkerw}) the rank of $\hat{\F}$ cannot exceed $p$ since $\B$ and $\C$ are of rank $p$ at most.
In fact, the low-rank property is indirectly promoted in~(\ref{eq:mcker}) through the kernel matrices. Since $\text{rank}(\F)\leq\min(\text{rank}(\K_x),\text{rank}(\K_y))$, we can limit the rank of $\hat{\F}$ by selecting rank deficient kernels.
On the other hand, the regularization terms in (\ref{eq:factkerw}) and (\ref{eq:mcker}) play a different role in each formulation. The regularization in (\ref{eq:factkerw}) promotes smoothness on the columns of the estimated factor matrices $\{\hat{\W},\hat{\H}\}$; or, in other words, similarity between the rows of $\{\hat{\W},\hat{\H}\}$ as measured by $\kapx$ and $\kapy$. On the contrary, the regularization in (\ref{eq:mcker}) promotes smoothness on $\hat{\v}$, which is tantamount to promoting similarity between the entries of $\hat{\F}$ in accordance with $\kapz$.

\subsection{KKMCEX error analysis}

In order to assess the performance of KKMCEX we will rely on the mean-square error
\begin{equation}\label{eq:risk}
  MSE := \ex{||\v - \hat{\v}_K||^2_2}
\end{equation}
where $\ex{\cdot}$ denotes the expectation with respect to $\e$. Before we proceed, we will outline Nyström's approximation.

\noindent \textbf{Definition 1}. Given a kernel matrix $\K$ and a binary sampling matrix $\S$ of appropriate dimensions,  the Nyström approximation~\cite{drineas2005nystrom} of $\K$ is $\T = \K\S^T(\S\K\S^T)^{\dagger}\S\K$, and the regularized Nyström approximation is
\begin{equation}\label{eq:nystreg}
  \tilde{\T} =  \K\S^T(\S\K\S^T + \mu\I)^{-1}\S\K.
\end{equation}
Nyström's approximation is employed to reduce the complexity of standard kernel regression problems such as the one in (\ref{eq:rralp}). Instead of $\K$, the low-rank approximation $\T$ is used to reduce the cost of inverting large-size matrices using the MIL~\cite{alaoui2015fast}. While it is known that the best low-rank approximation to a matrix is obtained from its top eigenvectors, Nyström's approximation is cheaper. Using Def. 1, 
the following lemma provides the bias and variance of the KKMCEX estimator in (\ref{eq:zKKMCEX}):
\begin{lemma}\label{th:biasvar} Given the kernel matrix $\K_z$ and its regularized Nyström approximation $\tilde{\T}_z$ with $\mu>0$, the MSE of the KKMCEX estimator is
    \begin{align}\label{eq:lem1}
        \text{MSE} &= \norm{(\K_z-\tilde{\T}_z)\gam}
        + \ex{{1\over\mu^2}\norm{(\K_z - \tilde{\T}_z)\S^T\bar{\e}}}
    \end{align}
    where the first term accounts for the bias and the second term accounts for the variance.
\end{lemma}
Lemma~\ref{th:biasvar} shows that the MSE of the KKMCEX can be expressed in terms of $\tilde{\T}_z$; see proof in the Appendix. Knowing that the 2-norm satisfies $\norm{\A}\leq\frob{\A}$, we have
\begin{align}
    &\norm{(\K_z-\tilde{\T}_z)\gam}+ \ex{{1\over\mu^2}\norm{(\K_z - \tilde{\T}_z)\S^T\bar{\e}}}
    \nonumber\\&\:\:\leq \frob{(\K_z-\tilde{\T}_z)}\left(\norm{\gam} +\ex{{1\over\mu^2}\norm{\S^T\bar{\e}}}\right).
\end{align}
Consequently, the upper bound on the MSE is proportional to the approximation error of $\tilde{\T}_z$ to $\K_z$. This suggests selecting $\{m_{i,j}\}_{(i,j)\in\Omega}$ so that this approximation error is minimized; see also~\cite{alaoui2015fast} where $\Omega$ is chosen according to the so-called leverage scores of $\K_z$ in order to minimize the regression error. The next theorem uses Lemma 1 to upper bound the MSE in \eqref{eq:lem1}; see the Appendix for its proof.
\begin{theorem}\label{th:mse}
    Let $\sigma_{NL}$ be the maximum eigenvalue of a nonsingular $\K_z$, and $\tilde{\gam}:=\L^T\gam$, where $\L$ is the eigenvector matrix of $\K_z-\tilde{\T}_z$. If $\e$ is a zero-mean vector of iid Gaussian random variables with covariance matrix $\nu^2\I$, the MSE of the KKMCEX estimator  is bounded as
    \begin{equation}\label{eq:msebound}
MSE \leq
\frac{\mu^2\sigma_{NL}^2}{(\sigma_{NL}+\mu)^2}\sum _{i=1}^{S}\tilde{\gam_i}^2 + \sigma^2_{NL}\sum_{i=S+1}^{NL}\tilde{\gam}_i^2+\frac{S\nu^2 \sigma_{NL}^2}{\mu^2}.
\end{equation}
\end{theorem}

Considering the right-hand side of (\ref{eq:msebound}), the first two terms correspond to the bias, while the last term is related to the variance. We observe that when $\M$ is fully observed, that is, $S=NL$, the bias can be made arbitrarily small  by having $\mu\rightarrow 0$. It is also of interest to assess how the MSE bound behaves as $S$ increases. Considering $\mu=S\mu'$ and fixed values in $(0,\infty)$ for $\mu'$, $||\tilde{\gam}_i||^2$ and $\sigma_{NL}$\footnote{Note that $||\tilde{\gam}_i||^2$ and $\sigma_{NL}$ depend on the selected kernel $\K_z$ and matrix $\F$, and do not depend on $\S$.}, the bias term reduces to
    \begin{align}\label{eq:biasboundtext} 
        \frac{S^2\mu'^2\sigma_{NL}^2}{(\sigma_{NL}+S\mu')^2}\sum _{i=1}^{S}\tilde{\gam_i}^2 + \sigma^2_{NL}\sum_{i=S+1}^{NL}\tilde{\gam}_i^2\;.
     \end{align}
We observe in \eqref{eq:biasboundtext} that as $S$ increases, terms move from the second summation to the first. Therefore, whether the bias term grows or diminishes depends on the multiplication factors in front of the two summations. Since  $\frac{S^2\mu'^2}{(\sigma_{NL}+S\mu')^2} \leq 1$ the bias term in \eqref{eq:biasboundtext} decreases with $S$. On the other hand, the variance term becomes $\frac{\nu^2 \sigma_{NL}^2}{S\mu'^2}$ and decays with $S$ as well. As a result, the MSE bound in Theorem~\ref{th:mse} decays up until $S=NL$.

\section{Ridge regression MCEX}
\label{sec:RRMCEX}

Although the KKMCEX algorithm is fast when $S$ is small, the size of the matrix to be inverted in (\ref{eq:gamest}) grows with $S$, hence increasing the computational cost.  Available approaches to reducing the computational cost of kernel regression methods are centered around the idea of approximating the kernel matrix. For instance,~\cite{alaoui2015fast} uses Nyström's approximation, that our performance analysis in Section IV was based on, whereas~\cite{yang2015randomized} relies on a sketch of $\K_z$ formed by a subset of its columns, hence reducing the number of regression coefficients; see also~\cite{avron2017faster}, where the kernel function is approximated by the inner product of random finite-dimensional feature maps, which also speeds up the matrix inversion. In this section, we reformulate the KKMCEX of Section~\ref{sec:kmr} to incorporate a low-rank approximation of $\K_z$ in order to obtain a reduced complexity estimate for $\bm \v$. Moreover, we also develop an online method based on this reformulation.

Recall from Eq. (\ref{eq:kerdot}) that a kernel can be viewed as the inner product of vectors mapped to a feature space $\mathcal{F}_z$, namely $\kapz((x_i,y_j),(x_n,y_l)) = \langle \phi_z(x_i,y_j),\phi_z(x_n,y_l)\rangle_{\mathcal{F}_z}$. Let $\tilde{\phi}_z:\mathcal{X}\times\mathcal{Y}\rightarrow\mathbb{R}^d$ be a feature map  approximating $\kappa_z$ so that 
    \begin{equation}\label{eq:kzfeatapprox}
\kapz((x_i,y_j),(x_n,y_l)) \simeq \langle \tilde{\phi}_z(x_i,y_j),\tilde{\phi}_z(x_n,y_l)\rangle.
    \end{equation}
Then, we define the $NL\times d$ feature matrix $\tbphi_z:=[\tilde{\phi}_z(x_1,y_1),\allowbreak\tilde{\phi}_z(x_2,y_1),\ldots,\tilde{\phi}_z(x_N,y_L)]^T$ and form $\tK_z=\tbphi_z\tbphi^T_z$. Note that $\tK_z$ is a rank-$d$ approximation of $\K_z$, and that the equality $\K_z=\tK_z$ is only feasible when $\text{rank}(\K_z)\leq d$. Consider $\tbphi_x=[\tilde{\phi}_x(x_1),\ldots,\tilde{\phi}_x(x_N)]$ and $\tbphi_y=[\tilde{\phi}_y(y_1),\ldots,\tilde{\phi}_y(y_L)]$, where $\tilde{\phi}_x:\mathcal{X}\rightarrow \mathbb{R}^{d_x}$ and $\tilde{\phi}_y:\mathcal{Y}\rightarrow \mathbb{R}^{d_y}$, as the feature matrices forming low-rank approximations to $\K_x$ and $\K_y$, respectively. Since  $\K_z=\K_y\otimes\K_x$ in KKMCEX, a prudent choice is $\tbphi_z=\tbphi_y\otimes \tbphi_x$. In the next section we will present means of constructing $\{\tbphi_x,\tbphi_y,\tbphi_z\}$ maps.

Since $\tK_z$ is a valid kernel matrix, upon replacing $\K_z$ in \eqref{eq:obsvec} with $\tK_z$, the observation model reduces to

\begin{equation}
    \bar{\m} = \S\tbphi_z\tbphi_z^T\bm\gam + \tilde{\e},
\end{equation}
where $\tilde{\e} = \bar{\e} + \S(\K_z-\tK_z)\gam$. With this model, the weights in (\ref{eq:mcker}) are obtained as
\begin{equation}\label{eq:mcfeat}
  \bm\hat{\gam} = \argmin_{\gam\in \mathbb{R}^{NL}} \norm{\bar{\m} - \S\tbphi_z\tbphi^T_z\gam} + \mu\gam^T\tbphi_z\tbphi^T_z\gam.
\end{equation}
Letting $\bxi:=\tbphi_z^T\bm\gamma$ and substituting into (\ref{eq:mcfeat}), we arrive at
\begin{equation}\label{eq:mcridge}
   \hxi=\argmin_{\bxi\in \mathbb{R}^d} \norm{\bar{\m} - \S\tbphi_z\bxi} + \mu\norm{\bxi}
\end{equation}
which admits the closed-form solution
\begin{equation}\label{eq:zeta}
  \hat{\bxi} = (\tbphi_z^T\S^T\S\tbphi_z+\mu\I)^{-1}\tbphi_z^T \S^T\bar{\m}.
\end{equation}
Using $\hat{\bxi}$, we obtain $\hat{\v}_{R}=\tbphi_z\hat{\bxi}$ as the \textit{ridge regression MCEX} (RRMCEX) estimate. Using the MIL~\eqref{eq:mil} on \eqref{eq:zeta}, it follows that 
\begin{align}\label{eq:zreqzk}
\hat{\bxi}=\tbphi_z^T \S^T(\S\tbphi_z\tbphi_z^T\S^T+\mu\I)^{-1}\bar{\m} 
\end{align}
and thus, 
\begin{equation}\label{eq:zreqzk2}
\hat{\v}_R=\tbphi_z\hat{\bxi}=\tK_z \S^T(\S\tK_z^T\S^T+\mu\I)^{-1}\bar{\m}.
\end{equation}
Therefore, \eqref{eq:zreqzk2} shows that $\hat{\v}_R$ is equivalent to the KKMCEX solution $\hat{\v}_K$ in \eqref{eq:zKKMCEX} after replacing $\K_z$ by its low-rank approximation $\tK_z$. For error-free approximation, $\K_z=\tbphi_z\tbphi_z^T$, while $\hat{\bxi}$ in \eqref{eq:zeta} can be viewed as the primal solution to the optimization problem in~\eqref{eq:mcridge}, and $\hat{\gam}$ in \eqref{eq:gamest} as its dual~\cite{shawe}. Still, obtaining $\hat{\bxi}$ requires multiplying two $d\times S$ matrices and inverting a $d\times d$ matrix, which incurs computational cost $\mathcal{O}(d^2S)$ when $S\geq d$, and  $\S\tbphi_z$ is obtained at cost $\mathcal{O}(dS)$. Thus, the cost of RRMCEX grows linearly with $S$ in contrast to KKMCEX that increases with $S^3$. 

By choosing an appropriate feature map so that $d\ll S$, it is possible to control the computational cost of calculating $\hat{\bxi}$. However, reduced computational cost by selecting a small $d$ might come at the price of an approximation error to $\K_z$, which correspondingly increases the estimation error of $\hat{\v}_R$. The selection of a feature matrix to minimize this error and further elaboration on the computational cost are given in Section \ref{sec:choosing}.

\subsection{Online RRMCEX}
\label{sec:oRRMCEX}
Online methods learn a model by processing one datum at a time. An online algorithm often results when the objective can be separated into several subfunctions, each depending on one or multiple data. In the context of MC, online implementation updates $\hat{\F}$ every time a new entry $\M_{i,j}$ becomes available. If we were to solve~\eqref{eq:gamest} each time a new observation was becoming available, inverting an $S\times S$ matrix per iteration would result in an overall prohibitively high computational cost. Still, the cost of obtaining an updated solution per observation can stay manageable using online kernel regression solvers that fall into three categories~\cite{van2014online}: dictionary learning, recursive regression and stochastic gradient descent based. Akin to~\cite{lu2016large,sheikholeslami2018}, we will pursue here the SGD.

Consider rewriting~\eqref{eq:mcridge} entrywise as
\begin{equation}\label{eq:mcrr}
    \hat{\bxi}=\argmin_{\bxi\in\mathbb{R}^d} \sum_{(i,j)\in \Omega}\left[m_{i,j} - \tilde{\phi}_z^T(x_i,y_j)\bxi\right]^2 + \mu\norm{\bxi}.
\end{equation}
With $n$ denoting each scalar observation, SGD iterations form a sequence of estimates
\begin{equation}\label{eq:sgd}
    \hat{\bxi}^n = \hat{\bxi}^{n-1} - t_n\left[\tilde{\phi}_z(x_i,y_j)(\tilde{\phi}^T_z(x_i,y_j)\hat{\bxi}^{n-1} - m_{i,j})+\mu\hat{\bxi}^{n-1}\right]
\end{equation}
where $t_n$ is the step size, $n=1,\ldots,S$ and the tuple $(i,j)$ denotes the indices of the entry revealed at iteration $n$. With properly selecting $t_n$, the sequence $\hat{\bxi}^n$ will converge to~\eqref{eq:mcrr} at per iteration cost $\mathcal{O}(d)$~\cite{bottou2012stochastic}. Apart from updating all entries in the matrix, ~\eqref{eq:sgd} can also afford a simple distributed implementation using e.g., the algorithms in~\cite{schizas}.

\noindent \textbf{Remark 3}.
Online algorithms for MC can be designed to solve the factorization-based formulation from (\ref{eq:mclag}) rewritten as
\begin{align}
\argmin_{\substack{\W\in\mathbb{R}^{N\times p}\\\H\in\mathbb{R}^{N\times p}}} \!\sum_{(i,j)\in\Omega} \!\left(\!(m_{i,j} \!-\! \w_i^T\h_j)^2 \!+\! {\mu \over |\Omega^w_i|}\norm{\w_i} \!+\! {\mu\over |\Omega^h_j|}\norm{\h_j}\!\right)
\end{align}
where $\w_i^T$ and $\h_j^T$ denote the $i^{\text{th}}$ and $j^{\text{th}}$ rows of $\H$ and $\W$ respectively, $\Omega^w_i=\{j\::\: (i,j) \in \Omega \}$, and $\Omega^h_j=\{i\::\: (i,j) \in \Omega \}$.
When $m_{i,j}$ becomes available, algorithms such as SGD and online ALS update the rows $\{\w_i^T$, $\h_j^T\}$ of the coefficient matrices.  This procedure can also be applied to the kernel MCEX formulation in~\eqref{eq:factker}, that solves for $\W$ and $\H$, although the rows $\{\w_i^T$, $\h_j^T\}$ cannot be updated independently due to the involvement of the kernel matrices~\cite{zhou2012}. Then, all entries in the $i^\text{th}$ row and $j^\text{th}$ column of $\hat{\F}$ are also updated per iteration, as opposed to our method which updates the whole matrix.
\section{Choosing the kernel matrices}
\label{sec:choosing}
In this section, we provide pointers on how to build matrices $\K_z$ for KKMCEX and $\tbphi_z$ for RRMCEX when prior information about either the matrix $\F$, or the input spaces $\mathcal{X}$ and $\mathcal{Y}$, is available.

\subsection{Kernels based on the graph Laplacian}
\label{sec:gsmc}
Suppose that the columns and rows of $\F$ lie on a graph, that is, each entry of a column or row vector is associated with a node on a graph that encodes the interdependencies with entries in the same vector. Specifically, we define an undirected weighted graph $\mathcal{G}_x = (\mathcal{X}, \mathcal{E}_x, \A_x)$ for the columns of $\F$, where $\mathcal{X}$ is the set of vertices with $|\mathcal{X}|=N$, $\mathcal{E}_x \subseteq \mathcal{X}\times\mathcal{X}$ is the set of edges connecting the vertices, and $\A_x \in \mathbb{R}^{N\times N}$ is a weighted adjacency matrix. Then, functions $\{f_l: \mathcal{X} \rightarrow \mathbb{R}\}^{L}_{l=1}$ are what is recently referred to as a graph signal~\cite{shuman2013emerging}, that is,  a map from the set $\mathcal{X}$ of vertices into the set of real numbers. Likewise, we define a graph $\mathcal{G}_y = (\mathcal{Y}, \mathcal{E}_y, \A_y)$ for the rows of $\F$, i.e., $\{g_n: \mathcal{Y} \rightarrow \mathbb{R}\}^{N}_{n=1}$, which are also graph signals. In a matrix of user-movie ratings for instance, we would have two graphs: one for the users and one for the movies. The graphs associated with the columns and rows yield the underlying structure of $\F$ that can be used to generate a pair of kernels.

Using $\A_x$ and $\A_y$, we can form the corresponding graph Laplacian as $\L_x:=\text{diag}(\A_x\bm 1)-\A_x$ and likewise for $\L_y$, that can serve as kernels. A family of graphical kernels results using a monotonic inverse function $r^\dagger(\cdot)$ on the Laplacian eigendecomposition as~\cite{smola2003kernels}
\begin{equation}\label{eq:kerlap}
        \K = \Q r^{\dagger}(\bm\Lambda)\Q^T.
\end{equation}
A possible choice of $r(\cdot)$ is the Gaussian radial basis function, which generates the diffusion kernel $r(\lambda_i)=e^{\eta\lambda_i}$, where $\lambda_i$ is the $i^\text{th}$ eigenvalue of $\L$, and $\eta$ a weight parameter. Alternatively, one can choose just the linear function $r(\lambda_i)=1 + \eta\lambda_i$, which results in the regularized Laplacian kernel.
By applying different weighting functions to the eigenvalues of $\L_x$ and $\L_y$, we promote smoother or more rapidly changing functions for the columns and rows of $\hat{\F}$~\cite{romerospace}. While  $\K_x$ and $\K_y$ are chosen as Laplacian kernels, this would not be the case for $\K_z = \K_y\otimes\K_x$ used in our KKMCEX context since it does not result from applying $r^\dagger(\cdot)$ to a Laplacian matrix. However, since $\K_x=\Q_x\Sig_x\Q_x^T$ and $\K_y=\Q_y\Sig_y\Q_y^T$, the eigendecomposition of $\K_z$ is $\K_z=(\Q_y\otimes\Q_x)(\Sig_y\otimes\Sig)(\Q^T_y \otimes\Q^T_x)$, and the notions of frequency and smoothness carry over. In other words, we are still promoting similarity among entries that are connected on the row and columns graphs through $\K_z$.

A key attribute in graph signal processing is that of ``graph bandlimitedness", which arises when a signal can be generated as a linear combination of a few eigenvectors of the Laplacian matrix. Therefore, a bandlimited graph signal belongs to an RKHS that is spanned by a bandlimited kernel~\cite{romero} that suppresses some of the frequencies of the graph. A bandlimited kernel is derived from the Laplacian matrix of a graph as in~\eqref{eq:kerlap}, using
\begin{equation}\label{eq:bandlimr}
    r(\lambda_i) = 0 \:~~ \forall i\notin\Psi,
\end{equation}
where $\Psi\subseteq\mathbb{N}$ contains the indices of frequencies not to be suppressed. As mentioned earlier, we define a graph for the columns and a graph for the rows of $\F$. Therefore, graph signals contained in the columns and rows may be bandlimited with different bandwidths. In order to form $\K_z$ our KKMCEX approach, we will need to apply different weighting functions akin to the one in~\eqref{eq:bandlimr} for kernel matrices $\K_x$ and $\K_y$.
\subsection{Kernels from known basis or features}
\label{sec:kerfeat}
In some applications the basis that spans the columns or rows of the unobserved matrix is assumed known, although this basis matrix needs not be a kernel. In order to be able to include such basis into the kernel framework, we need to generate kernel functions that span the same spaces as the columns and rows of $\F$.

Consider the input sets $\{\mathcal{X},\mathcal{Y}\}$ whose entries can be mapped into an Euclidean space through feature extraction functions $\theta_x:\mathcal{X}\rightarrow\mathbb{R}^{t_x}$ and $\theta_y:\mathcal{Y}\rightarrow\mathbb{R}^{t_y}$ such that $\theta_x(x_i):=\x_i$ and $\theta_y(y_j):=\y_j$. For instance, in a movie recommender system where the users are represented in $\mathcal{X}$ and the movies in $\mathcal{Y}$, each coordinate of $\y_j$ could denote the amount of action, drama and nudity in the movie, and $\x_i$ would contain weights according to the user's preference for each attribute. We may then use the feature vectors to determine the similarities among entries in $\mathcal{X}$ and $\mathcal{Y}$ by means of kernel functions.

Let $\X :=[\x_1,\ldots,\x_N]^T$ and $\Y :=[\y_1,\ldots,\y_L]^T$. If  $\text{span}(\F)\subseteq\text{span}(\X)$ and $\text{span}(\F^T)\subseteq\text{span}(\Y)$, we may conveniently resort to the linear kernel. The linear kernel amounts to the dot product in Euclidean spaces, which we use to define the pair $\kappa_x(x_i,x_j) = \x_i^T\x_j$ and $\kappa_y(y_i,y_j) = \y_i^T\y_j$. This leads to a straightforward construction of the kernel matrices for KKMCEX as $\K_x=\X\X^T$ and $\K_y=\Y\Y^T$.

Besides the linear kernel, it is often necessary to use a different kernel class for each $\kappa_x$ and $\kappa_y$ chosen to better fit the spaces spanned by the rows and columns of $\F$. For instance, the Gaussian kernel defined as $\kapx(x_i,x_j) = \text{exp}\{-\norm{\x_i-\x_j}/(2\eta)\}$, is a widely used alternative in the regression of smooth functions. 
\subsection{Feature maps for RRMCEX}

Aiming to construct $\tbphi_z$ that approximates $\K_z$ at reduced complexity, we choose $\tilde{\phi}_z$ with $d\ll S$. To approximate linear kernels, let $\tilde{\phi}_x(x_i)=\x_i$ and $\tilde{\phi}_y(y_j)=\y_j$ so that we can set $\tilde{\phi}_z(x_i,y_j)=\tilde{\phi}_y(y_j)\otimes\tilde{\phi}_x(x_i)$ and $\tbphi_z=\Y\otimes\X$. Note that in this case $\tbphi_z\tbphi_z^T$ yields a zero-error approximation to $\K_z=(\Y\otimes\X)(\Y\otimes\X)^T$, which renders the KKMCEX and RRMCEX solutions equivalent. 

On occasion, $\X$ and $\Y$ may have large column dimension, thus rendering $\Y\otimes\X$ undesirable as a feature matrix in RRMCEX. In order to overcome this hurdle, we build an approximation to the column space of $\Y\otimes\X$ from the SVD of $\X$ and $\Y$. Consider the SVDs of matrices $\X = \U_x\D_x\V_x^T$ and $\Y = \U_y\D_y\V_y^T$, to obtain $\Y\otimes\X = (\U_y\otimes\U_x)(\D_y\otimes\D_x)(\V_y^T\otimes\V_x^T)$. Let $\tbphi_z=\U_d\D_d$, where $\U_d$ and $\D_d$ respectively hold the top $d$ singular vectors and singular values of $\Y\otimes\X$. The SVD has cost $\mathcal{O}(Nt_x^2)$ for $\X$ and $\mathcal{O}(Lt_y^2)$ for $\Y$. Comparatively, the cost of building $\K_x$ and $\K_y$ for the linear kernel is $\mathcal{O}(N^2t_x)$ and $\mathcal{O}(L^2t_y)$, respectively. Therefore, choosing RRMCEX over KKMCEX in this case incurs no extra overhead.

When a function other than the linear kernel is selected, obtaining an approximation is more complex. To approximate a Gaussian kernel on $\mathcal{X}\times\mathcal{X}$, the vectors $\{\tilde{\phi}_x(\x_i)\}^N_{i=1}$ can be obtained by means of Taylor series expansion~\cite{cotter2011explicit} or random Fourier features~\cite{avron2017faster}, which can also approximate Laplacian, Cauchy and polynomial kernels~\cite{avron2017faster,rahimi2008random}. Therefore, the maps $\tilde{\phi}_x$ and $\tilde{\phi}_y$ must be designed according to the chosen kernels.

In some instances, such as when dealing with Laplacian kernels, $\X$ and $\Y$ are not available and we are only given $\K_x$ and $\K_y$. We are then unable to derive approximations to the kernel matrices by means of maps $\tilde{\phi}_x$ and $\tilde{\phi}_y$. Nevertheless, we can still derive an adequate $\tbphi_z$ to approximate $\K_z$. Indeed, Mercer's Theorem asserts that there are eigenfunctions $\{q_n\}^{NL}_{n=1}$ in $\mathcal{H}_z$ along with a sequence of nonnegative real numbers $\{\sigma_n\}^{NL}_{n=1}$, such that
\begin{equation}\label{eq:mercer}
    \kapz((x_i,y_j),(x_n,y_l)) = \sum^{NL}_{n=1}\sigma_nq_n(x_i,y_j)q_n(x_n,y_l).
\end{equation}
We can find \eqref{eq:mercer} from the eigendecomposition $\K_z=\bm\Q_z\Sig_z\bm\Q_z^T$, where $q_n$ is the $n\th$ eigenvector in $\Q_z$ and $\sigma_n$ the $n\th$ eigenvalue in $\Sig_z$. If $\K_z$ is low rank, we can construct $\tbphi_z=\Q_d\Sig^{{1\over 2}}_d$, where $\Q_d$ and $\Sig_d$ respectively hold the top $d$ eigenvectors and eigenvalues of $\K_z$. Note that, since $\K_z=(\Q_y\otimes\Q_x)(\Sig_y\otimes\Sig_x)(\Q^T_y \otimes\Q^T_x)$, we only need to eigendecompose smaller matrices $\K_x$ and $\K_y$ at complexity $\mathcal{O}(N^3+L^3)$. In some cases however, such as when using Laplacian kernels, the eigendecompositions are readily available, and $\bphi_z$ can be built at a markedly reduced cost.

\section{Numerical tests}

In this section, we test the performance of the KKMCEX, RRMCEX and online  (o)RRMCEX algorithms developed in Sections~\ref{sec:kmr},~\ref{sec:RRMCEX} and~\ref{sec:oRRMCEX}, respectively; and further compare them to the solution of \eqref{eq:factker} obtained with ALS~\cite{jain2013low} and SGD~\cite{zhou2012}. We run the tests on synthetic and real datasets, with and without noise, and measure the signal-to-noise-ratio (SNR) as $\frac{\frob{\F}}{\frob{\E}}$. The algorithms are run until convergence over $N_{r}=50$ realizations with different percentages of observed entries, denoted by $P_{s}=100S/(NL)$, which are taken uniformly at random per realization. As figure of merit, we use
\begin{equation}
\text{NMSE} = \frac{1}{N_{r}}\sum\limits_{i=1}^{N_{r}}\frac{\frob{\hat{\F}_i - \F}}{\frob{\F}}
\end{equation}
where $\hat{\F}_i$ is the estimate at realization $i$. We show results for the optimal combination of regularization and kernel parameters, found via grid search. Finally, ALS and SGD are initialized by a product of two random factor matrices, an both are stopped.
\subsection{Synthetic data}
\label{sec:synth}

 We first test the algorithms on synthetic data. The $250\times250$ data matrix is generated as $\F = \K_x\bm \Gamma\K_y$, where $\bm \Gamma$ is a $250\times250$ matrix of Gaussian random deviates. For $\K_x$ and $\K_y$ we use Laplacian diffusion kernels with $\eta=1$ based on Erdös-Rényi graphs, whose binary adjacency matrices are unweighted and any two vertices are connected with probability 0.03. The resulting $\F$ is approximately low-rank, with the sum of the first 10 eigenvalues accounting for 96\% of the total eigenvalue sum. Therefore, we set the rank bound $p$ to 10 for the ALS and SGD algorithms. Whether $\F$ is approximately low-rank or exactly low-rank did not affect our results, as they were similar for matrices with an exact rank of 10.
For KKMCEX, $\K_z=\K_y\otimes \K_x$, and for RRMCEX $\tbphi_z=\Q_d\Sig_d^{1\over 2}$, where $\Q_d$ contains the top 250 eigenvectors of $\K_z$, and $\Sig_d$ the corresponding top 250 eigenvalues. 

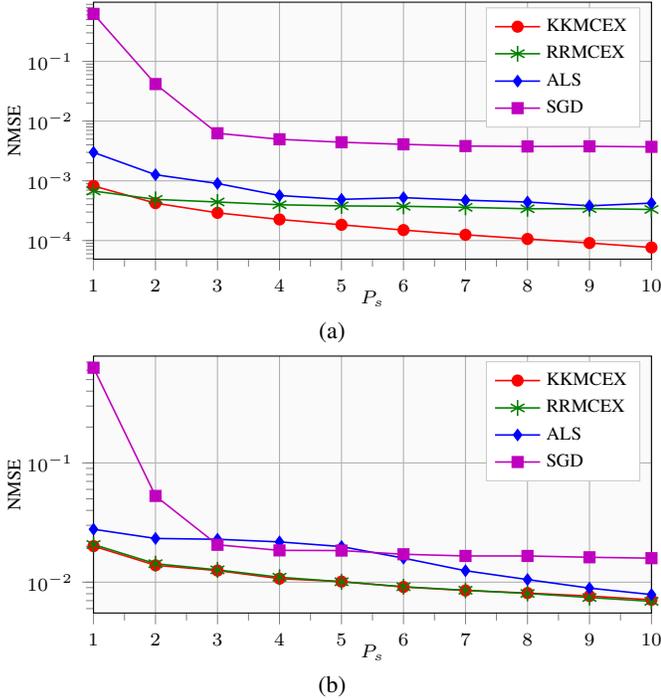
\begin{figure}[h]
    \centering
    \begin{subfigure}
     {\input{paperplots/synth_2}}
     \vspace{-0.5cm}
     \caption*{(a)}
     \end{subfigure}
    \vspace{0.3cm} 
    \begin{subfigure}
     {\input{paperplots/synth_1}}
     \vspace{-0.5cm}
     \caption*{(b)}
     \end{subfigure}
     \vspace{0.2cm}
    \caption{NMSE vs $P_s$ for (a) synthetic noiseless matrix; and (b) synthetic noisy matrix.}
    \label{fig:erd}
\end{figure}

Fig.~\ref{fig:erd} shows the simulated NMSE when $\M$ is noiseless (a) or noisy (b). We deduce from Fig.~\ref{fig:erd}a that all algorithms except SGD achieve a very small NMSE, below 0.003 at $P_s=1\%$ that falls to 0.007 at $P_s=10\%$. Of the three algorithms, KKMCEX has the smallest error except at $P_s=1\%$, where RRMCEX performs best. Although the error drops below 0.005 for SGD at $P_s > 4\%$, it is outperformed by the other algorithms by an order of magnitude. Fig.~\ref{fig:erd}b shows the same results when Gaussian noise is added to $\F$ at $\text{snr}=1$. We observe that KKMCEX and RRMCEX are matched and attain the lowest error, whereas ALS and SGD have larger errors across $P_s$. This corroborates that thanks to the regularization term that smoothes over all the entries instead of row or column-wise, the noise effect is reduced. Interestingly, RRMCEX is able to reduce the noise effect despite the bias it suffers because it only uses the top 250 eigenvalues of $\K_z$ from a total of $62{,}500$. This is mainly due to the additive noise being evenly distributed across the eigenspace of $\K_z$. Therefore, by keeping only the eigenvectors associated with the top 250 eigenvalues in $\K_z$, we are discarding those dimensions in which the SNR is lower.

\begin{figure}[h]
    \centering
    \begin{subfigure}
     {\input{paperplots/synth_time_2}}
     \vspace{-0.5cm}
     \caption*{(a)}
     \end{subfigure}
    \vspace{0.3cm} 
    \begin{subfigure}
     {\input{paperplots/synth_time_1}}
     \vspace{-0.5cm}
     \caption*{(b)}
     \end{subfigure}
     \vspace{0.2cm}
    \caption{Time vs $P_s$ for (a) synthetic noiseless matrix; and (b) synthetic noisy matrix.}
    \label{fig:erdtime}
\end{figure}
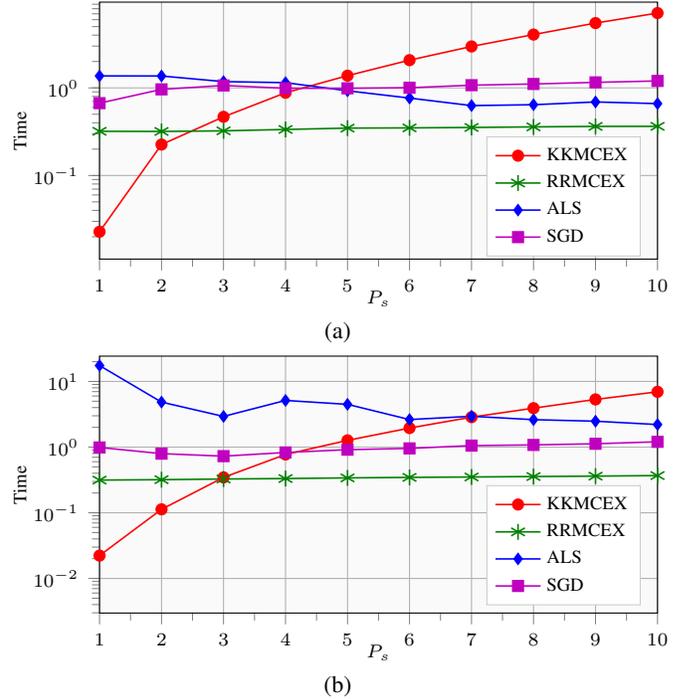

Fig.~\ref{fig:erdtime} depicts the time needed for the algorithms to perform the simulations reported in Fig.~\ref{fig:erd}. We observe in Fig.~\ref{fig:erdtime}a that RRMCEX has an almost constant computation time, whereas the time for KKMCEX grows with $P_s$ as expected since the size of the matrix to be inverted increases with $S$. On the other hand, ALS and SGD require less time than KKMCEX for the larger values of $P_s$, but are always outperformed by RRMCEX. Moreover, the ALS time is reduced as $P_s$ increases because the number of iterations required to converge to the minimum is smaller. Fig.~\ref{fig:erdtime}b suggests that the noise only impacts ALS, which has its computation time rise considerably across all $P_s$. Overall, Figures~\ref{fig:erd} and~\ref{fig:erdtime} illustrate that RRMCEX has the best performance for the synthetic matrix both in terms of NMSE and computational cost.

\subsection{Temperature measurements}
\label{sec:real}
In this case, $\F$ has size $150\times365$ comprising temperature readings taken by 150 stations over 365 days in 2002 in the United States\footnote{http://earthpy.org/ulmo.html}. The columns and rows of $\F$ are modeled as graph signals with $\A_x$ and $\A_y$ for the graphs formed by the stations and the days of the year, respectively. We use  Laplacian diffusion kernels both for $\K_x$ and $\K_y$, while $\K_z$ and $\tbphi_z$ are obtained as in the tests on synthetic data, except that $\tbphi_z$ is constructed with the top 150 eigenvectors of $\K_z$. The matrix $\A_x$ is obtained as in~\cite{chen}, where a graph $\mathcal{G}$ with unweighted adjacency matrix $\P$ is generated for the stations, and each station is a vertex connected to the 8 geographically closest stations. Next, we obtain the undirected graph $\mathcal{G}'$ with symmetric adjacency matrix $\P’ =\text{sign}(\P^T+\P)$. Finally, the entries of $\A_x$ are constructed as $(\A_x)_{i,j}= \text{exp}(-\frac{N^2 d_{i,j}}{\sum_{i,j} d_{i,j}})$, where $\{d_{i,j}\}$ are geodesic distances on $\mathcal{G}$. We adopt a graph on which each day is a vertex and each day is connected to the 10 past and future days to form $\A_y$.

\begin{figure}[h]
    \centering
    \begin{subfigure}
     {\input{paperplots/real_1}}
     \vspace{-0.5cm}
     \caption*{(a)}
     \end{subfigure}
     \vspace{0.3cm}
    \begin{subfigure}
     {\input{paperplots/realnoisy_1}}
     \vspace{-0.5cm}
     \caption*{(b)}
     \end{subfigure}
     \vspace{0.2cm}
    \caption{NMSE vs $P_s$ for the matrix of temperature measurements (a) without noise, and (b) with noise.}
    \label{fig:time}
\end{figure}
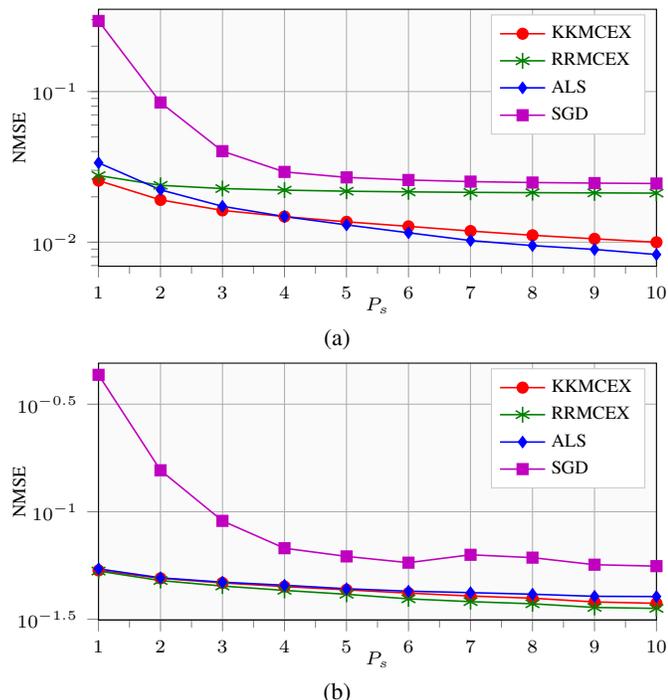

Fig.~\ref{fig:time} shows the simulated tests for (a) the matrix of temperature readings, and (b) the same matrix with additive Gaussian noise at $\text{snr}=1$. Fig.~\ref{fig:time}a demonstrates that KKMCEX achieves the lowest error for the first three data points, while afterwards ALS has a slight edge over KKMCEX. The real data matrix $\F$ is approximately low rank, since the sum of the first 10 singular values accounts for 75\% of the total sum. This explains why RRMCEX fares worse than KKMCEX. Because it only contains the top 150 eigenvectors of $\K_z$, which is full rank, the vectorized data $\m$ lies in part outside the space spanned by $\tbphi_z$. Indeed, increasing the number of eigenvectors in $\tbphi_z$ results in a lower error, although the computational cost increases accordingly. Fig.~\ref{fig:time}b further demonstrates that the addition of noise has the least impact on RRMCEX, which attains the lowest error slightly below KKMCEX. On the other hand, ALS has a marginally higher error whereas the gap between SGD and the other three methods remains. Fig.~\ref{fig:timetime} depicts the computational time for the results in Fig.~\ref{fig:time}a, which are similar to those obtained for the synthetic dataset.
\begin{figure}[h]
    \centering
    \input{paperplots/real_time_1}
    \caption{Time vs $P_s$ for the noiseless matrix of temperature measurements.}
    \label{fig:timetime}
\end{figure}
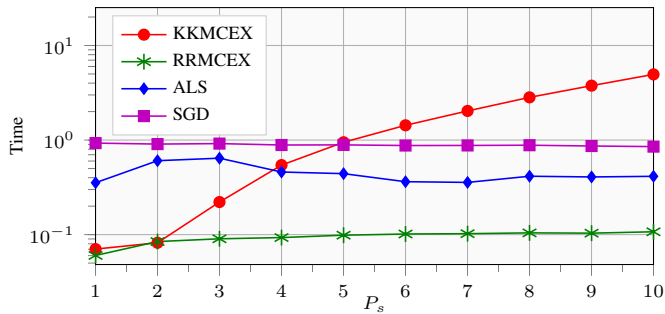

\subsection{Mushroom dataset}
The Mushroom dataset\footnote{http://archive.ics.uci.edu/ml} comprises 8,124 labels and as many feature vectors. Each label indicates whether a sample is edible or poisonous, and each vector has 22 entries describing the shape, color, etc. of the mushroom sample. After removing items with missing features, we are left with 5,643 labels and feature vectors. Here, we solve a clustering problem in which $\F$ is a $5,643\times 5,643$ adjacency matrix, where $\F_{i,j} = 1$ if the $i^{\text{th}}$ and the $j^{\text{th}}$ mushroom samples belong to the same class (poisonous or edible), and $\F_{i,j} = -1$ otherwise.

We encode the matrix stacking the feature vectors via one-hot encoding to produce a $5,643\times 98$ binary feature matrix analogous to $\X$ in Section \ref{sec:kerfeat}. We build the kernel matrix $\K_x$ from the Pearson correlation coefficients of the rows of $\X$, and let $\K_y=\K_x$. The feature matrix $\tbphi_z$ for RRMCEX is built using the top 3,000 left singular vectors of $\X\otimes\X$. 

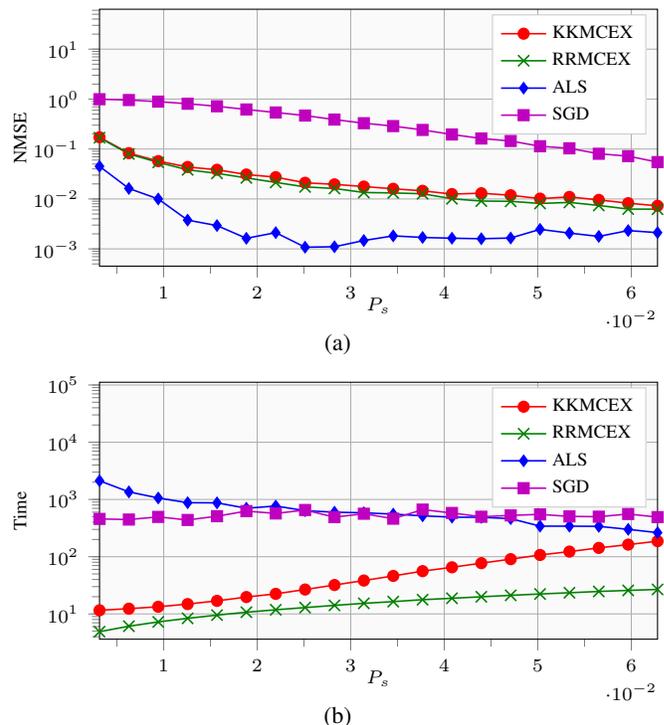
\begin{figure}[h]
    \centering
    \begin{subfigure}
     {\input{paperplots/mushroom_1}}
     \vspace{-0.5cm}
     \caption*{(a)}
     \end{subfigure}
     \vspace{0.3cm}
    \begin{subfigure}
     {\input{paperplots/mushroom_time_1}}
     \vspace{-0.5cm}
     \caption*{(b)}
     \end{subfigure}
     \vspace{0.2cm}
     \caption{results for the mushroom adjacency matrix as (a) NMSE vs $P_s$, and (b) time vs $P_s$.}
    \label{fig:mushroom}
\end{figure}

Fig.~\ref{fig:mushroom}a shows the test results on the mushroom adjacency matrix from $S=2,000$ $(P_s=0.006\%)$ to $S=20{,}000$ $(P_s=0.036\%)$ in steps of 1,000 observations. KKMCEX and RRMCEX achieve similar NMSE, while SGD has an error one order of magnitude higher, and ALS outperforms both by around one order of magnitude. This difference with ALS is because regression-based methods restrict the solution to belong to the space spanned by the basis matrix. On the other hand, when solving \eqref{eq:factker}, we do not enforce the constraints $\W\in\mathcal{H}_x$, $\H\in\mathcal{H}_y$~\cite{bazerque2013,zhou2012}. Therefore, when the prior information encoded in the kernel matrices is imperfect, ALS might be able to find a factorization that fits better the data at the cost of having $\hat{\W}\notin\mathcal{H}_x$ and $\hat{\H}\notin\mathcal{H}_y$. However, in Fig.~\ref{fig:mushroom}b we see that the computational cost for ALS and SGD is much higher than for KKMCEX and RRMCEX for the smaller $\P_s$. On the other hand, the time for ALS decreases with $S$ due to requiring les iterations to converge, whereas for KKRRMCEX and RRMCEX it increases with $S$.

\subsection{Online MC}

\begin{figure}[h]
    \centering
    \begin{subfigure}
     {\input{paperplots/ol_erdos}}
     \vspace{-0.5cm}
     \caption*{(a)}
     \vspace{0.2cm}
     \end{subfigure}
    \begin{subfigure}
     {\input{paperplots/ol_temp}}
     \vspace{-0.5cm}
     \caption*{(b)}
     \end{subfigure}
     \vspace{0.2cm}
     \caption{NMSE vs time for the online algorithms on the (a) synthetic noiseless matrix; and (b) matrix of noiseless temperature measurements. Each mark denotes 1000 iterations have passed.}
    \label{fig:sg}
\end{figure}
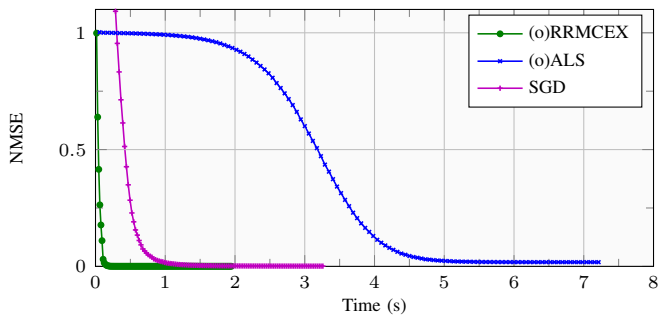
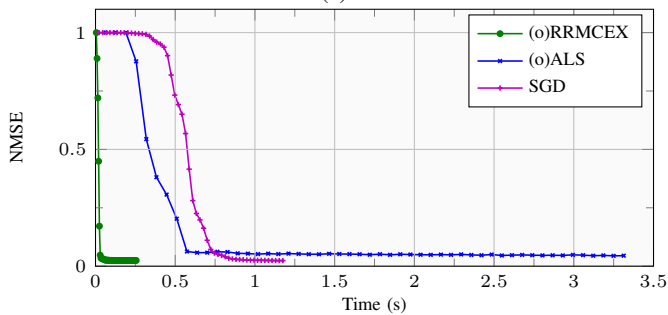

In the online scenario, we compare the (o)RRMCEX algorithm with online (o)ALS and SGD. One observation is revealed per iteration at random, and all three algorithms process a single observation per iteration in a circular fashion. Per realization, we run tests on both synthetic and temperature matrices with $P_s=10\%$, that is, $S=6{,}250$ and $S=5{,}475$ observations for the synthetic and temperature matrices, respectively, for a single realization. 

Fig.~\ref{fig:sg}a depicts the tests for the noiseless synthetic matrix. Clearly, (o)RRMCEX converges much faster than SGD and (o)ALS. Indeed, as opposed to SGD and (o)ALS, which require several passes over the data, (o)RRMCEX approaches the minimum in around $6{,}000$ iterations. Moreover, it achieves the smallest NMSE of 0.0004, which is slightly below the 0.0011 obtained by SGD. Fig.~\ref{fig:sg}b shows the results for the temperature matrix without noise. Again, we observe that (o)RRMCEX converges the fastest to the minimum, whereas SGD requires many passes through the data before it starts descending, while (o)ALS converges much faster than with the synthetic data. Regarding the NMSE, (o)RRMCEX and SGD achieve the same minimum value.

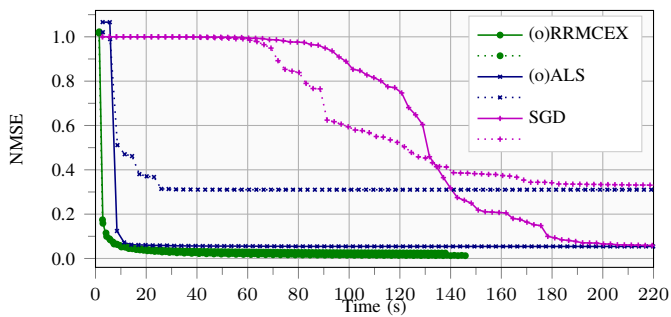
\begin{figure}[t]
    \centering
     {\input{paperplots/online_1}}
     \vspace{-0.5cm}
     \caption{NMSE vs time for the mushroom adjacency matrix. Solid lines denote $S=20{,}000$, and dotted lines denote $S=10{,}000$. Each mark denotes 10000 iterations have passed.}
    \label{fig:mushonline}
\end{figure}
The tests on the Mushroom dataset are run with $S=10{,}000$ $(P_s=0.033\%)$ and $S=20{,}000$ $(P_s=0.036\%)$ observations following the same procedure as with the synthetic and temperature datasets. Fig.~\ref{fig:mushonline} shows results for the Mushroom adjacency matrix with the error for $S=20{,}000$ plotted in solid lines, and for $S=10{,}000$ in dotted lines. We observe that for $S=20{,}000$, (o)RRMCEX crosses the minimum of (o)ALS and SGD in 7 seconds, whereas (o)ALS and SGD converge to this minimum in 12 and 200 seconds, respectively. Afterwards, the line for (o)RRMCEX keeps descending until an error of $0.012$ is reached. For $S=10{,}000$ the convergence time of (o)RRMCEX and SGD remains almost unchanged, whereas for (o)ALS it increases to 26 seconds. Moreover, the error of both (o)ALS and SGD grows much larger, whereas (o)RRMCEX exhibits just a small increase.

\section{Conclusions}
In this paper, we have taken a comprehensive look at MC under the framework of RKHS. We have viewed columns and rows of the data matrix as functions from an RKHS, and leveraged kernel theory to account for the available prior information on the contents of the sought matrix. Moreover, we have developed two estimation algorithms that offer simplicity and speed as their main advantages. When the number of observed data is small, KKMCEX obtains the full matrix estimate by inverting a reduced-size matrix thanks to the Representer Theorem. On the other hand, when the number of observations is too large for KKMCEX to handle, RRMCEX can be employed instead in order to lower the computational cost with no impact on the recovery error when noise is present. In addition, RRMCEX can be easily turned into an online method implemented via SGD iterations. Compared to mainstream methods designed for the factorization-based formulation, namely ALS and SGD, our RRMCEX exhibited improved performance in simulated and real data sets.

Our future research agenda includes improving both KKMCEX and RRMCEX through parallel and accelerated regression methods, as well as designing robust sampling strategies for MCEX formulated as a kernel regression. 
\appendix

\subsection{Proof of Lemma 1}
For the KKMCEX estimator (\ref{eq:zKKMCEX}),  the MSE is given as
\begin{equation}\label{eq:risk}
  MSE := \ex{\norm{\v - \hz_K}} = \ex{\norm{\v - \K_z\hgam}}.
\end{equation}
Plugging the estimator from (\ref{eq:gamest}) into (\ref{eq:risk}) yields 
\begin{align}
  MSE &= \ex{\norm{\v-\K_z \S^T (\S\K_z\S^T + \mu\I)^{-1}(\S\v+\bar{\e})}}\nonumber \\
  &= \norm{(\I - \K_z \S^T (\S\K_z\S^T + \mu\I )^{-1}\S)\v}\nonumber \\ &+ \ex{\norm{\K_z\S^T(\S\K_z\S^T+\mu\I)^{-1}\bar{\e}}} \label{eq:biasvar}
\end{align}
where we have used that $\mathbb{E}\{\e\}=\bm 0$. Further, the first and second terms in (\ref{eq:biasvar}) are the bias and variance of the KKMCEX estimator, respectively. If we substitute $\v=\K_z\gam$ into the first term of (\ref{eq:biasvar}), we obtain
\begin{align}
bias &= \norm{(\I - \K_z\S^T (\S\K_z\S^T + \mu\I)^{-1}\S )\K_z\gam}\nonumber \\
&= \norm{(\K_z - \K_z \S^T (\S\K_z\S^T + \mu\I)^{-1}\S\K_z )\gam}\nonumber \\
& = \norm{(\K_z-\tilde{\T}_z)\gam}\label{eq:bias}
\end{align}
where $\tilde{\T}_z$ is the regularized Nyström approximation of $\K_z$ in~(\ref{eq:nystreg}).
On the other hand, the variance term is
\begin{align}
var &= \ex{\norm{\K_z\S^T(\S\K_z\S^T+\mu\I)^{-1}\bar{\e}}}\nonumber \\
&\mspace{-25mu}= \mathbb{E}\{{1\over\mu^2}\left|\left|\K_z\S^T(\S\K_z\S^T+\mu\I)^{-1}\right.\right.\nonumber\\&\left.\left. \:\:\:(\mu\I+\S\K_z\S^T-\S\K_z\S^T)\bar{\e}\right|\right|_2^2\}\nonumber \\
&\mspace{-25mu}= \ex{{1\over\mu^2}\norm{\K_z\S^T - \K_z\S^T(\S\K_z\S^T+\mu\I)^{-1}\S\K_z\S^T\bar{\e}}}\nonumber\\
&\mspace{-25mu}= \ex{{1\over\mu^2}\norm{(\K_z - \K_z\S^T(\S\K_z\S^T+\mu\I)^{-1}\S\K_z)\S^T\bar{\e}}}\nonumber\\
&\mspace{-25mu}= \ex{{1\over\mu^2}\norm{(\K_z - \tilde{\T}_z)\S^T\bar{\e}}}.\label{eq:quad}
\end{align}
Adding the two terms in~\eqref{eq:bias} and~\eqref{eq:quad}, we obtain the MSE in~\eqref{eq:lem1}.

\subsection{Proof of Theorem 2}
Since $\K_z - \tilde{\T}_z$ appears in the bias and variance terms in Lemma 1, we will first derive an upper bound on its eigenvalues that will eventually lead us to a bound on the MSE. To this end, we will need a couple of lemmas.

\begin{lemma}\label{lem:sim}
Given a symmetric matrix $\A\in\mathbb{R}^{N\times N}$ and a symmetric nonsingular matrix $\B\in\mathbb{R}^{N\times N}$, it holds that
$\lambda_k(\A\B) = \lambda_k(\B^{1\over 2}\A\B^{1\over 2})$; and also $\lambda_k(\A\B) \leq \lambda_k(\A)\lambda_N(\B)$.
\end{lemma}
\begin{proof}
Since $\B$ is invertible and symmetric, we can write $\A\B = \B^{-{1\over 2}}(\B^{1\over 2}\A\B^{1\over 2})\B^{1\over 2}$.
Therefore, $\A\B$ is similar to $\B^{{1\over 2}}\A\B^{1\over 2}$, and they both share the same eigenvalues.
Let $\mathcal{U}\subset\mathbb{R}^{N}\setminus\{\bm 0\}$. From the min-max theorem~\cite{lax}, the $k^{th}$ eigenvalue of $\A$ satisfies
\begin{equation}
    \lambda_k(\A) = \min_{\mathcal{U}} \left\{ \max_{\x\in \mathcal{U}} \frac{\x^T\A\x}{\x^T\x} \:|\: \text{dim}(\mathcal{U})=k \right\}.
\end{equation}
Therefore, we have
\begin{align}
    \lambda_k(\A\B) &= \lambda_k(\B^{1\over 2}\A\B^{1\over 2}) \nonumber\\&\mspace{-50mu}=
      \min_{\mathcal{U}} \left\{ \max_{\x\in \mathcal{U}} \frac{\x^T\B^{1\over 2}\A\B^{1\over 2}\x}{\x^T\x} \:|\: \text{dim}(\mathcal{U})=k \right\} \nonumber\\
    &\mspace{-50mu}= \min_{\mathcal{U}} \left\{ \max_{\x\in \mathcal{U}} \frac{\x^T\B^{1\over 2}\A\B^{1\over 2}\x}{\x^T\B^{1\over 2}\B^{1\over 2}\x} \frac{\x^T\B\x}{\x^T\x} \:|\: \text{dim}(\mathcal{U})=k \right\}\nonumber \\
    &\mspace{-50mu}\leq \min_{\mathcal{U}} \left\{ \max_{\x\in \mathcal{U}} \frac{\x^T\A\x}{\x^T\x} \:|\: \text{dim}(\mathcal{U})=k \right\}\lambda_N(\B)\nonumber \\
    &\mspace{-50mu}= \lambda_k(\A)\lambda_N(\B).
\end{align}
\end{proof}
The following lemma bounds the eigenvalues of $\K_z - \tilde{\T}_z$, and the regularized Nystrom approximation $\tilde{\T}_z$ in~(\ref{eq:nystreg}).
\begin{lemma}\label{lem:nyst} With $\K_z$ as in~\eqref{eq:Kz} and $\tilde{\T}_z$ as in~\eqref{eq:nystreg}, the eigenvalues of $\K_z-\tilde{\T}_z$ are bounded as
\begin{equation}
    \K_z - \tilde{\T}_z \preceq \frac{\mu\sigma_{NL}}{\sigma_{NL}+\mu}\I_S' + \sigma_{NL}\I_S
\end{equation}
where $\sigma_{NL}$ is the largest eigenvalue of $\K_z$, $\I_S := \text{diag}([0,\allowbreak 0,\ldots,1,1])$ has $S$ zeros on its diagonal, and $\I'_S := \I-\I_S$.
\end{lemma}
\begin{proof}

Using the eigendecomposition $\K_z=\Q_z\Sig_z\Q_z^T$, we can write 
\begin{align}
    \K_z-\tilde{\T}_z &= \K_z - \K_z\S^T(\S\K_z\S^T+\mu\I)^{-1}\S\K_z \nonumber \\
    &= \Q_z\Sig_z^{1\over 2}\left[\I-\Sig_z^{1\over 2}
    \Q_z^T\S^T(\S\Q_z\Sig_z^{1\over 2}\Sig_z^{1\over 2}
    \Q_z^T\S^T \right.\nonumber \\&\left.\:\:\:\:+\mu \I)^{-1}\S\Q_z\Sig_z^{1\over 2}\right]
    \Sig_z^{1\over 2}\Q_z^T.\label{eq:midmat}
\end{align}
Applying the MIL to the matrix inside the square brackets of~\eqref{eq:midmat}, we arrive at
\begin{align}
    &\I-\Sig_z^{1\over 2}
    \Q_z^T\S^T(\S\Q_z\Sig_z^{1\over 2}\Sig_z^{1\over 2}
    \Q_z^T\S^T+\mu \I)^{-1}\S\Q_z\Sig_z^{1\over 2} \nonumber\\&\:\:= (\I + {1\over\mu}\Sig_z^{1\over 2}\Q_z^T\S^T\S\Q_z\Sig_z^{1\over 2})^{-1}.
\end{align}
That in turn implies
\begin{align}\label{eq:kzp}
    &\K_z-\tilde{\T}_z = \mu\Q_z\Sig_z^{1\over 2}(\mu\I + \Sig_z^{1\over 2}\Q_z^T\S^T\S\Q_z\Sig_z^{1\over 2})^{-1}\Sig_z^{1\over 2}\Q_z^T \nonumber \\
    &= \mu\Q_z\Sig_z^{1\over 2}(\Sig_z + \mu\I  - \Sig_z + \Sig_z^{1\over 2}\Q_z^T\S^T\S\Q_z\Sig_z^{1\over 2})^{-1}\Sig_z^{1\over 2}\Q_z^T \nonumber \\
  &= \mu\Q_z\Sig_z^{1\over 2}\left[(\Sig_z+\mu\I)^{1\over2}\left(\I -  (\Sig_z+\mu\I)^{-{1\over2}}\Sig_z(\Sig_z+\mu\I)^{-{1\over 2}}\right.\right.\nonumber\\& \left.\left. \:\:\:\:+ (\Sig_z+\mu\I)^{-{ 1\over 2}}\Sig_z^{1\over 2}\Q_z^T\S^T\S\Q_z\Sig_z^{1\over 2}(\Sig_z+\mu\I)^{-{1\over 2}})\right)\right.\nonumber \\&\left. \:\:\:\:(\Sig_z+\mu\I)^{1\over2}\right]^{-1}\Sig_z^{1\over 2}\Q_z^T\nonumber \\
  &= \mu\Q_z\Sig_z^{1\over 2}(\Sig_z+ \mu\I)^{-{1\over2}}(\I-\P)^{-1}(\Sig_z+ \mu\I)^{- {1\over2}}\Sig_z^{1\over 2}\Q_z^T 
\end{align}
where
\begin{align}\label{eq:P}
    \P =&\Sig_z(\Sig_z+\mu\I)^{-1} \\&- (\Sig_z+\mu\I)^{-{ 1\over 2}}\Sig^{1\over 2}\Q_z^T\S^T\S\Q_z\Sig^{1\over 2}(\Sig_z+\mu\I)^{-{1\over 2}}.\nonumber 
\end{align}
Regarding the eigenvalues of $\K_z-\tilde{\T}_z$ in~\eqref{eq:kzp}, we can bound them as
\begin{align}
    &\lambda(\K_z-\tilde{\T}_z) \nonumber\\&= \mu\,\lambda(\Q_z\Sig_z^{1\over 2}(\Sig_z+ \mu\I)^{-{1\over2}}(\I-\P)^{-1}(\Sig_z+ \mu\I)^{- {1\over2}}\Sig_z^{1\over 2}\Q_z^T)\nonumber\\
     &= \mu\,\lambda(\Sig_z^{1\over 2}(\Sig_z+ \mu\I)^{-{1\over2}}(\I-\P)^{-1}(\Sig_z+ \mu\I)^{- {1\over2}}\Sig_z^{1\over 2})\nonumber\\
     &= \mu\,\lambda((\I-\P)^{-1}(\Sig_z+ \mu\I)^{-1}\Sig_z)\nonumber\\
     &\leq\frac{\mu\sigma_{NL}}{\sigma_{NL}+\mu}\lambda((\I-\P)^{-1})\label{eq:kzbound1}
 \end{align}
where $\lambda(\cdot)$ denotes the eigenvalues of a matrix, and we have applied Lemma~\ref{lem:sim} on the third equality and the last inequality. Knowing that $\lambda(\I-\P)=\I-\lambda(\P)$ we can now bound the eigenvalues of $\P$ as
\begin{align}\label{eq:C2}
    \lambda(\P) &\!=\!\lambda(\Sig_z^{1\over 2}(\Sig_z\!+\!\mu\I)^{- {1\over 2}}\Q_z^T(\I \!-\! \S^T\S)\Q_z(\Sig_z\!+\!\mu\I)^{- {1\over 2}}\Sig_z^{1\over 2})\nonumber\\
 &=\lambda(\Q_z^T(\I - \S^T\S)\Q_z(\Sig_z+\mu\I)^{- {1}}\Sig_z)\nonumber\\
 &\leq \frac{\sigma_{NL}}{\sigma_{NL}+\mu} \lambda(\Q_z^T(\I - \S^T\S)\Q_z)\nonumber\\
 &= \frac{\sigma_{NL}}{\sigma_{NL}+\mu}\lambda(\I_S)
\end{align}
where we have applied Lemma~\ref{lem:sim} on the second and third inequalities, and $\I_S:=\text{diag}{[0,0,\ldots,1,1]}$ has $S$ zeros on its diagonal. Next, we have that $\I-\P\succeq \I - \frac{\sigma_{NL}}{\sigma_{NL}+\mu}\I_S$, and thus
\begin{equation}\label{eq:iqbound}
(\I-\P)^{-1}\preceq  \I_S' + \frac{\sigma_{NL}+\mu}{\mu}\I_S
\end{equation}
where $\I_S':=\I-\I_S$. Finally, combining~\eqref{eq:iqbound} with~\eqref{eq:kzbound1} yields
\begin{equation}\label{eq:kt}
    \K_z - \tilde{\T}_z \preceq \frac{\mu\sigma_{NL}}{\sigma_{NL}+\mu}\I_S' + \sigma_{NL}\I_S
\end{equation}
which concludes the proof.
\end{proof}
Using Lemmas~\ref{lem:sim} and~\ref{lem:nyst}, we can proceed to establish a bound on the bias and variance. Considering the eigendecomposition $\K_z-\tilde{\T}_z=\L\bm\Lambda\L^T$,  we can write the bias in (\ref{eq:bias}) as
\begin{align}
bias &= \norm{\L\bm\Lambda\L^T\gam}.
\end{align}
With $\tilde{\gam}:=\L^T\gam$, and using Lemma~\ref{lem:nyst} the bias is bounded as
\begin{align}
bias
 &  = \norm{\L\bm\Lambda\tilde{\gam}}
   = \tilde{\gam}^T\bm\Lambda^2\tilde{\gam}\nonumber\\
&\leq \frac{\mu^2\sigma_{NL}^2}{(\sigma_{NL}+\mu)^2}\tilde{\gam}^T\I_S'\tilde{\gam} + \sigma_{NL}^2\tilde{\gam}^T\I_S\tilde{\gam}\nonumber\\
&= \frac{\mu^2\sigma_{NL}^2}{(\sigma_{NL}+\mu)^2}\sum _{i=1}^{S}\tilde{\gam_i}^2 + \sigma^2_{NL}\sum_{i=S+1}^{NL}\tilde{\gam}_i^2.\label{eq:biasbound}
\end{align}

To bound the variance in (\ref{eq:quad}), recall that $\e$ is a Gaussian random vector with covariance matrix $\nu^2\I$, while $\bar{\e}$ has covariance matrix $\nu^2\S\S^T$. Then (\ref{eq:quad}) is a quadratic form in $\e$, whose the variance becomes
\begin{align}\label{eq:var1}
    var &=\ex{{1\over\mu^2}\norm{(\K_z - \tilde{\T}_z)\S^T\bar{\e}}} \nonumber\\&= {\nu^2\over \mu^2}\trace{\S(\K_z - \tilde{\T}_z)^2\S^T}\nonumber\\&= {\nu^2\over \mu^2}\trace{(\K_z - \tilde{\T}_z)^2\S^T\S}.
\end{align}
The matrix inside the trace in~\eqref{eq:var1} has $NL-S$ zero entries in its diagonal. Lemma~\ref{lem:nyst}, on the other hand, implies that diagonal entries of $\K_z-\tilde{\T}_z$ are smaller than its largest eigenvalue; that is, $\left[\K_z-\tilde{\T}_z\right]_{i,i}\leq\sigma_{NL}$. Coupling this with~\eqref{eq:var1} yields
\begin{align}
var
    \leq \frac{S\nu^2 \sigma_{NL}^2}{\mu^2}.\label{eq:varbound}
\end{align}
 Finally, combining the bias bound in (\ref{eq:biasbound}) with the variance bound in (\ref{eq:varbound}), yields the bound for the MSE as
\begin{equation}
MSE \leq
\frac{\mu^2\sigma_{NL}^2}{(\sigma_{NL}+\mu)^2}\sum _{i=1}^{S}\tilde{\gam_i}^2 + \sigma^2_{NL}\sum_{i=S+1}^{NL}\tilde{\gam}_i^2+\frac{S\nu^2 \sigma_{NL}^2}{\mu^2}.
\end{equation}



{\bibliographystyle{IEEEtran}
\input{bibliography.bbl}}

\end{document}

%% file: paperplots/synth_2.tex
\begin{tikzpicture}

\definecolor{color0}{rgb}{0.75,0,0.75}

\begin{axis}[
height=\figH,
legend cell align={left},
legend entries={{KKMCEX},{RRMCEX},{ALS},{SGD}},
legend style={draw=white!80.0!black},
tick align=outside,
tick pos=left,
width=\figW,
x grid style={white!69.01960784313725!black},
xlabel={$P_s$},
xmajorgrids,
xmin=1, xmax=10,
y grid style={white!69.01960784313725!black},
ylabel={NMSE},
ymajorgrids,
ymin=4.88382165331581e-05, ymax=0.978615644932848,
ymode=log
]
\addplot [semithick, red, mark=*, mark size=2, mark options={solid}]
table [row sep=\\]{%
1	0.000821778847218432 \\
2	0.000422937705752729 \\
3	0.000291298774249317 \\
4	0.000225722830678218 \\
5	0.000183603082933908 \\
6	0.000149820761155694 \\
7	0.000124906794686742 \\
8	0.000106360117961479 \\
9	9.09586518638552e-05 \\
10	7.66123068812802e-05 \\
};
\addplot [semithick, green!50.0!black, mark=asterisk, mark size=3, mark options={solid}]
table [row sep=\\]{%
1	0.000675674983581312 \\
2	0.000488924217807112 \\
3	0.000442192429540376 \\
4	0.000397695192502633 \\
5	0.000380280922031144 \\
6	0.000374925823724514 \\
7	0.000359854728160683 \\
8	0.000340762766342931 \\
9	0.000342518965682444 \\
10	0.000331627220550889 \\
};
\addplot [semithick, blue, mark=diamond*, mark size=2, mark options={solid}]
table [row sep=\\]{%
1	0.00299140309290216 \\
2	0.00126552460923503 \\
3	0.000902782176175691 \\
4	0.000567784199019005 \\
5	0.000489489885751583 \\
6	0.000523251238464411 \\
7	0.00047390134608686 \\
8	0.000441769652456598 \\
9	0.000381510118197297 \\
10	0.000423437901984349 \\
};
\addplot [semithick, color0, mark=square*, mark size=2, mark options={solid}]
table [row sep=\\]{%
1	0.623840277307257 \\
2	0.0417191236980314 \\
3	0.00625594167720573 \\
4	0.00496380865108259 \\
5	0.00442688866176499 \\
6	0.00408757675406379 \\
7	0.00382310706938721 \\
8	0.00376319928530136 \\
9	0.00378410177668199 \\
10	0.00370831272356546 \\
};
\end{axis}

\end{tikzpicture}

%% file: paperplots/synth_1.tex
\begin{tikzpicture}

\definecolor{color0}{rgb}{0.75,0,0.75}

\begin{axis}[
height=\figH,
legend cell align={left},
legend entries={{KKMCEX},{RRMCEX},{ALS},{SGD}},
legend style={draw=white!80.0!black},
tick align=outside,
tick pos=left,
width=\figW,
x grid style={white!69.01960784313725!black},
xlabel={$P_s$},
xmajorgrids,
xmin=1, xmax=10,
y grid style={white!69.01960784313725!black},
ylabel={NMSE},
ymajorgrids,
ymin=0.00551200936582286, ymax=0.787596513930837,
ymode=log
]
\addplot [semithick, red, mark=*, mark size=2, mark options={solid}]
table [row sep=\\]{%
1	0.0200607465590559 \\
2	0.0138518129666094 \\
3	0.0125122454674508 \\
4	0.0107077453891651 \\
5	0.0101105052560168 \\
6	0.00910750642633654 \\
7	0.00851517103687704 \\
8	0.00810307705811838 \\
9	0.00764943857383004 \\
10	0.00711368365750933 \\
};
\addplot [semithick, green!50.0!black, mark=asterisk, mark size=3, mark options={solid}]
table [row sep=\\]{%
1	0.0206982003288518 \\
2	0.0142875259722693 \\
3	0.0126793455948814 \\
4	0.0110049243471191 \\
5	0.0100773053600709 \\
6	0.00915508721561409 \\
7	0.00855127356151305 \\
8	0.00803848641826675 \\
9	0.00743576032319739 \\
10	0.00690659853560005 \\
};
\addplot [semithick, blue, mark=diamond*, mark size=2, mark options={solid}]
table [row sep=\\]{%
1	0.0277920921785841 \\
2	0.0232584725727422 \\
3	0.0228952271720093 \\
4	0.0217598909296638 \\
5	0.0198902875754209 \\
6	0.0159101841080542 \\
7	0.0124731809733922 \\
8	0.0105174994278346 \\
9	0.00891446704377881 \\
10	0.00787125904922002 \\
};
\addplot [semithick, color0, mark=square*, mark size=2, mark options={solid}]
table [row sep=\\]{%
1	0.628564023071457 \\
2	0.0528724524567559 \\
3	0.0205817504061724 \\
4	0.0184788318965524 \\
5	0.0184117711367433 \\
6	0.0171403748459506 \\
7	0.0165918213568707 \\
8	0.0165883261899603 \\
9	0.0161716059373092 \\
10	0.0158917231875577 \\
};
\end{axis}

\end{tikzpicture}

%% file: paperplots/synth_time_2.tex
\begin{tikzpicture}

\definecolor{color0}{rgb}{0.75,0,0.75}

\begin{axis}[
height=\figH,
legend cell align={left},
legend entries={{KKMCEX},{RRMCEX},{ALS},{SGD}},
legend style={at={(0.97,0.01)}, anchor=south east, draw=white!80.0!black},
tick align=outside,
tick pos=left,
width=\figW,
x grid style={white!69.01960784313725!black},
xlabel={$P_s$},
xmajorgrids,
xmin=1, xmax=10,
y grid style={white!69.01960784313725!black},
ylabel={Time},
ymajorgrids,
ymin=0.0111275791521016, ymax=9.51646567322877,
ymode=log
]

\addplot [semithick, red, mark=*, mark size=2, mark options={solid}]
table [row sep=\\]{%
1	0.02282751816 \\
2	0.22575990032 \\
3	0.4674190852 \\
4	0.87934864384 \\
5	1.37819069788 \\
6	2.07198395664 \\
7	2.96970280124 \\
8	4.05846853884 \\
9	5.49073797576 \\
10	7.14024266344 \\
};
\addplot [semithick, green!50.0!black, mark=asterisk, mark size=3, mark options={solid}]
table [row sep=\\]{%
1	0.31967489748 \\
2	0.3183728088 \\
3	0.32348011452 \\
4	0.3353346496 \\
5	0.34757577024 \\
6	0.34951128928 \\
7	0.35352574192 \\
8	0.35897556868 \\
9	0.36348748496 \\
10	0.3640883488 \\
};
\addplot [semithick, blue, mark=diamond*, mark size=2, mark options={solid}]
table [row sep=\\]{%
1	1.369231 \\
2	1.366542 \\
3	1.179424 \\
4	1.146785 \\
5	0.926719 \\
6	0.765323 \\
7	0.627878 \\
8	0.642418 \\
9	0.690996 \\
10	0.66148 \\
};
\addplot [semithick, color0, mark=square*, mark size=2, mark options={solid}]
table [row sep=\\]{%
1	0.669476 \\
2	0.961692 \\
3	1.065656 \\
4	0.989634 \\
5	0.988957 \\
6	1.003926 \\
7	1.073133 \\
8	1.10814 \\
9	1.155521 \\
10	1.200103 \\
};
\end{axis}

\end{tikzpicture}

%% file: paperplots/synth_time_1.tex
\begin{tikzpicture}

\definecolor{color0}{rgb}{0.75,0,0.75}

\begin{axis}[
height=\figH,
legend cell align={left},
legend entries={{KKMCEX},{RRMCEX},{ALS},{SGD}},
legend style={at={(0.97,0.03)}, anchor=south east, draw=white!80.0!black},
tick align=outside,
tick pos=left,
width=\figW,
x grid style={white!69.01960784313725!black},
xlabel={$P_s$},
xmajorgrids,
xmin=1, xmax=10,
y grid style={white!69.01960784313725!black},
ylabel={Time},
ymajorgrids,
ymin=0.0029717959653632, ymax=24.3026160942703,
ymode=log
]
\addplot [semithick, red, mark=*, mark size=2, mark options={solid}]
table [row sep=\\]{%
1	0.0222845424 \\
2	0.11282208668 \\
3	0.3456869832 \\
4	0.7720588242 \\
5	1.26668563328 \\
6	1.95039019188 \\
7	2.8593453606 \\
8	3.91640782088 \\
9	5.32485186532 \\
10	6.95027293684 \\
};
\addplot [semithick, green!50.0!black, mark=asterisk, mark size=3, mark options={solid}]
table [row sep=\\]{%
1	0.31465842644 \\
2	0.31917709908 \\
3	0.32681671564 \\
4	0.33212138476 \\
5	0.33972369216 \\
6	0.34594686692 \\
7	0.35074560256 \\
8	0.35642319636 \\
9	0.3615836954 \\
10	0.36873499976 \\
};
\addplot [semithick, blue, mark=diamond*, mark size=2, mark options={solid}]
table [row sep=\\]{%
1	17.472369 \\
2	4.843284 \\
3	2.93265 \\
4	5.138004 \\
5	4.478572 \\
6	2.627292 \\
7	2.95029 \\
8	2.617679 \\
9	2.479773 \\
10	2.206412 \\
};
\addplot [semithick, color0, mark=square*, mark size=2, mark options={solid}]
table [row sep=\\]{%
1	0.987693 \\
2	0.795266 \\
3	0.728335 \\
4	0.828534 \\
5	0.913801 \\
6	0.955808 \\
7	1.054415 \\
8	1.079694 \\
9	1.12484 \\
10	1.204513 \\
};
\end{axis}

\end{tikzpicture}

%% file: paperplots/real_1.tex
\begin{tikzpicture}

\definecolor{color0}{rgb}{0.75,0,0.75}

\begin{axis}[
height=\figH,
legend cell align={left},
legend entries={{KKMCEX},{RRMCEX},{ALS},{SGD}},
legend style={draw=white!80.0!black},
tick align=outside,
tick pos=left,
width=\figW,
x grid style={white!69.01960784313725!black},
xlabel={$P_s$},
xmajorgrids,
xmin=1, xmax=10,
y grid style={white!69.01960784313725!black},
ylabel={NMSE},
ymajorgrids,
ymin=0.00692964866433607, ymax=0.351656571733216,
ymode=log
]
\addplot [semithick, red, mark=*, mark size=2, mark options={solid}]
table [row sep=\\]{%
1.00091324200913	0.0256307065484197 \\
2	0.0190901149720963 \\
2.99908675799087	0.0162601468499377 \\
4	0.0147931756312273 \\
5.00091324200913	0.0136539600260737 \\
6	0.0127408998799112 \\
7.00091324200913	0.0118621886071249 \\
8	0.0111164779965638 \\
9.00091324200913	0.0105285210481886 \\
10	0.00999251182923801 \\
};
\addplot [semithick, green!50.0!black, mark=asterisk, mark size=3, mark options={solid}]
table [row sep=\\]{%
1.00091324200913	0.0277026372205054 \\
2	0.0238578794189882 \\
2.99908675799087	0.0227024115242863 \\
4	0.022163921561617 \\
5.00091324200913	0.0218008848768301 \\
6	0.021577638030536 \\
7.00091324200913	0.0214216535968899 \\
8	0.0213223569874865 \\
9.00091324200913	0.021246723910568 \\
10	0.0211733041074308 \\
};
\addplot [semithick, blue, mark=diamond*, mark size=2, mark options={solid}]
table [row sep=\\]{%
1.00091324200913	0.0336548136977337 \\
2	0.0222660908772749 \\
2.99908675799087	0.0173154667694766 \\
4	0.0147970087362766 \\
5.00091324200913	0.0130292059812877 \\
6	0.0115503230884679 \\
7.00091324200913	0.010257179470636 \\
8	0.00949402711619182 \\
9.00091324200913	0.00895296688868364 \\
10	0.00828380257574766 \\
};
\addplot [semithick, color0, mark=square*, mark size=2, mark options={solid}]
table [row sep=\\]{%
1.00091324200913	0.294171242051376 \\
2	0.0845997728305092 \\
2.99908675799087	0.040205885400084 \\
4	0.0292762788104468 \\
5.00091324200913	0.0269300569310588 \\
6	0.0258771015600159 \\
7.00091324200913	0.0252457157271043 \\
8	0.024888228178419 \\
9.00091324200913	0.0246794848552846 \\
10	0.0245176190926173 \\
};
\end{axis}

\end{tikzpicture}

%% file: paperplots/realnoisy_1.tex
\begin{tikzpicture}

\definecolor{color0}{rgb}{0.75,0,0.75}

\begin{axis}[
height=\figH,
legend cell align={left},
legend entries={{KKMCEX},{RRMCEX},{ALS},{SGD}},
legend style={draw=white!80.0!black},
tick align=outside,
tick pos=left,
width=\figW,
x grid style={white!69.01960784313725!black},
xlabel={$P_s$},
xmajorgrids,
xmin=1, xmax=10,
y grid style={white!69.01960784313725!black},
ylabel={NMSE},
ymajorgrids,
ymin=0.0313629501346448, ymax=0.490575587881076,
ymode=log
]
\addplot [semithick, red, mark=*, mark size=2, mark options={solid}]
table [row sep=\\]{%
1.00091324200913	0.0533607731770825 \\
2	0.0490399990859721 \\
2.99908675799087	0.0466812044273482 \\
4	0.0449024659259294 \\
5.00091324200913	0.0433605064988201 \\
6	0.0418751141594174 \\
7.00091324200913	0.0405609801682441 \\
8	0.0396354487296585 \\
9.00091324200913	0.0380813868718206 \\
10	0.0375324336086269 \\
};
\addplot [semithick, green!50.0!black, mark=asterisk, mark size=3, mark options={solid}]
table [row sep=\\]{%
1.00091324200913	0.0529927693298436 \\
2	0.0478801061875746 \\
2.99908675799087	0.0451442077924408 \\
4	0.0430977643762603 \\
5.00091324200913	0.041325809832392 \\
6	0.0393717604238542 \\
7.00091324200913	0.038205815162746 \\
8	0.037344169283947 \\
9.00091324200913	0.0359091572832942 \\
10	0.0355388010616597 \\
};
\addplot [semithick, blue, mark=diamond*, mark size=2, mark options={solid}]
table [row sep=\\]{%
1.00091324200913	0.0542688071885271 \\
2	0.0493024634306788 \\
2.99908675799087	0.0470546790064023 \\
4	0.0455395997695119 \\
5.00091324200913	0.0438479558249749 \\
6	0.0427281940209701 \\
7.00091324200913	0.0420676217513413 \\
8	0.041362954755327 \\
9.00091324200913	0.0404458547307689 \\
10	0.0403429951529492 \\
};
\addplot [semithick, color0, mark=square*, mark size=2, mark options={solid}]
table [row sep=\\]{%
1.00091324200913	0.432932379268893 \\
2	0.155890394551292 \\
2.99908675799087	0.090697578827736 \\
4	0.0676796779470154 \\
5.00091324200913	0.0620063260457825 \\
6	0.0579802637183204 \\
7.00091324200913	0.0630941676751302 \\
8	0.0611822287379813 \\
9.00091324200913	0.0567354144135076 \\
10	0.0558682089938228 \\
};
\end{axis}

\end{tikzpicture}

%% file: paperplots/real_time_1.tex
\begin{tikzpicture}

\definecolor{color0}{rgb}{0.75,0,0.75}

\begin{axis}[
height=\figH,
legend cell align={left},
legend entries={{KKMCEX},{RRMCEX},{ALS},{SGD}},
legend style={at={(0.03,0.97)}, anchor=north west, draw=white!80.0!black},
tick align=outside,
tick pos=left,
width=\figW,
x grid style={white!69.01960784313725!black},
xlabel={$P_s$},
xmajorgrids,
xmin=1, xmax=10,
y grid style={white!69.01960784313725!black},
ylabel={Time},
ymajorgrids,
ymin=0.048286544059414, ymax=25.16823439682803,
ymode=log
]
\addplot [semithick, red, mark=*, mark size=2, mark options={solid}]
table [row sep=\\]{%
1.00091324200913	0.07032311658 \\
2	0.08204411374 \\
2.99908675799087	0.22107592278 \\
4	0.54363951514 \\
5.00091324200913	0.9478805363 \\
6	1.43248152978 \\
7.00091324200913	2.03227154466 \\
8	2.82915905488 \\
9.00091324200913	3.76160159184 \\
10	4.94787113828 \\
};
\addplot [semithick, green!50.0!black, mark=asterisk, mark size=3, mark options={solid}]
table [row sep=\\]{%
1.00091324200913	0.0601961356 \\
2	0.08446870902 \\
2.99908675799087	0.0903367905 \\
4	0.09310100248 \\
5.00091324200913	0.09865545922 \\
6	0.10152996146 \\
7.00091324200913	0.10235950096 \\
8	0.10435171948 \\
9.00091324200913	0.10363987188 \\
10	0.10719268776 \\
};
\addplot [semithick, blue, mark=diamond*, mark size=2, mark options={solid}]
table [row sep=\\]{%
1.00091324200913	0.35386 \\
2	0.60497 \\
2.99908675799087	0.64333 \\
4	0.459354 \\
5.00091324200913	0.441354 \\
6	0.362905 \\
7.00091324200913	0.356698 \\
8	0.414517 \\
9.00091324200913	0.407419 \\
10	0.413637 \\
};
\addplot [semithick, color0, mark=square*, mark size=2, mark options={solid}]
table [row sep=\\]{%
1.00091324200913	0.928816 \\
2	0.906927 \\
2.99908675799087	0.917652 \\
4	0.886666 \\
5.00091324200913	0.889183 \\
6	0.875888 \\
7.00091324200913	0.877353 \\
8	0.882201 \\
9.00091324200913	0.865291 \\
10	0.852829 \\
};
\end{axis}

\end{tikzpicture}

%% file: paperplots/mushroom_1.tex
\begin{tikzpicture}

\definecolor{color0}{rgb}{0.75,0,0.75}
\begin{axis}[
height=\figH,
legend cell align={left},
legend entries={{KKMCEX},{RRMCEX},{ALS},{SGD}},
legend style={draw=white!80.0!black},
tick align=outside,
tick pos=left,
width=\figW,
x grid style={white!69.01960784313725!black},
xlabel={$P_s$},
xmajorgrids,
xmin=0.00314036334443546, xmax=0.0628072668887092,
y grid style={white!69.01960784313725!black},
ylabel={NMSE},
ymajorgrids,
ymin=0.00045102187752741, ymax=63,
ymode=log
]
\addplot [semithick, red, mark=*, mark size=2, mark options={solid}]
table [row sep=\\]{%
0.00314036334443546	0.171989905495545 \\
0.00628072668887092	0.0827146913629449 \\
0.00942109003330638	0.0572777103211728 \\
0.0125614533777418	0.04366067210594 \\
0.0157018167221773	0.0384377221230591 \\
0.0188421800666128	0.0308616371483135 \\
0.0219825434110482	0.0273252169996439 \\
0.0251229067554837	0.0210825126783188 \\
0.0282632700999191	0.0196152091173045 \\
0.0314036334443546	0.0177188532305446 \\
0.0345439967887901	0.0159715229519306 \\
0.0376843601332255	0.0145651199308738 \\
0.040824723477661	0.012525576687825 \\
0.0439650868220964	0.0130156420847947 \\
0.0471054501665319	0.0118623487562207 \\
0.0502458135109674	0.0101551875830621 \\
0.0533861768554028	0.0109860756369488 \\
0.0565265401998383	0.00957240613961455 \\
0.0596669035442737	0.008230241762541 \\
0.0628072668887092	0.00728893352996385 \\
};
\addplot [semithick, green!50.0!black, mark=x, mark size=3, mark options={solid}]
table [row sep=\\]{%
0.00314036334443546	0.168634170249586 \\
0.00628072668887092	0.0796306204142648 \\
0.00942109003330638	0.0539633253922965 \\
0.0125614533777418	0.0383289636747577 \\
0.0157018167221773	0.032575818027752 \\
0.0188421800666128	0.0264052301620971 \\
0.0219825434110482	0.0214530279053629 \\
0.0251229067554837	0.0174363775732962 \\
0.0282632700999191	0.0162099777571205 \\
0.0314036334443546	0.0134770828373522 \\
0.0345439967887901	0.0131994244718906 \\
0.0376843601332255	0.0127198030590217 \\
0.040824723477661	0.0100740092569746 \\
0.0439650868220964	0.00907838846225483 \\
0.0471054501665319	0.00896591320871053 \\
0.0502458135109674	0.00813273712907167 \\
0.0533861768554028	0.00852616184886254 \\
0.0565265401998383	0.00741557863282963 \\
0.0596669035442737	0.00627432034764827 \\
0.0628072668887092	0.00623819360773389 \\
};\addplot [semithick, blue, mark=diamond*, mark size=2, mark options={solid}]
table [row sep=\\]{%
0.00314036334443546	0.0450136994407415 \\
0.00628072668887092	0.0162311463622119 \\
0.00942109003330638	0.0099976785217729 \\
0.0125614533777418	0.00376917659469899 \\
0.0157018167221773	0.00293144494694876 \\
0.0188421800666128	0.00163002429072057 \\
0.0219825434110482	0.0021153961870307 \\
0.0251229067554837	0.00108069761627402 \\
0.0282632700999191	0.0011038177126906 \\
0.0314036334443546	0.00147130118895051 \\
0.0345439967887901	0.001823057009197567 \\
0.0376843601332255	0.001694722356244698 \\
0.040824723477661	0.00163547197508415 \\
0.0439650868220964	0.001598572709062814 \\
0.0471054501665319	0.001650244341227013 \\
0.0502458135109674	0.00246442317977105 \\
0.0533861768554028	0.00207246547303119 \\
0.0565265401998383	0.00177824239340247 \\
0.0596669035442737	0.00232600532242899 \\
0.0628072668887092	0.00211734857978917 \\
};
\addplot [semithick, color0, mark=square*, mark size=2, mark options={solid}]
table [row sep=\\]{%
0.00314036334443546	0.993552  \\
0.00628072668887092	0.957036  \\
0.00942109003330638	0.88936   \\
0.0125614533777418	0.807248  \\
0.0157018167221773	0.718215  \\
0.0188421800666128	0.618946  \\
0.0219825434110482	0.538834  \\
0.0251229067554837	0.467937  \\
0.0282632700999191	0.388986  \\
0.0314036334443546	0.328859  \\
0.0345439967887901	0.28783   \\
0.0376843601332255	0.241137  \\
0.040824723477661	0.195646  \\
0.0439650868220964	0.162546  \\
0.0471054501665319	0.144713  \\
0.0502458135109674	0.113649  \\
0.0533861768554028	0.10337   \\
0.0565265401998383	0.0803813 \\
0.0596669035442737	0.0719788  \\
0.0628072668887092	0.0547096  \\
};
\end{axis}

\end{tikzpicture}

%% file: paperplots/mushroom_time_1.tex
\begin{tikzpicture}

\definecolor{color0}{rgb}{0.75,0,0.75}

\begin{axis}[
height=\figH,
legend cell align={left},
legend entries={{KKMCEX},{RRMCEX},{ALS},{SGD}},
legend style={draw=white!80.0!black},
tick align=outside,
tick pos=left,
width=\figW,
x grid style={white!69.01960784313725!black},
xlabel={$P_s$},
xmajorgrids,
xmin=0.00314036334443546, xmax=0.0628072668887092,
y grid style={white!69.01960784313725!black},
ylabel={Time},
ymajorgrids,
ymin=3.62445354153321, ymax=111695.18556762265,
ymode=log
]
\addplot [semithick, red, mark=*, mark size=2, mark options={solid}]
table [row sep=\\]{%
0.00314036334443546	11.584255443 \\
0.00628072668887092	12.3575701782 \\
0.00942109003330638	13.3410106978 \\
0.0125614533777418	14.8621302286 \\
0.0157018167221773	16.9242811348 \\
0.0188421800666128	19.7342306078 \\
0.0219825434110482	22.4208252164 \\
0.0251229067554837	26.6854753586 \\
0.0282632700999191	31.9960121924 \\
0.0314036334443546	38.3203455758 \\
0.0345439967887901	46.1262642784 \\
0.0376843601332255	55.9057545986 \\
0.040824723477661	65.2809044538 \\
0.0439650868220964	77.1407983798 \\
0.0471054501665319	90.563341461 \\
0.0502458135109674	107.1701975606 \\
0.0533861768554028	122.7738599972 \\
0.0565265401998383	142.4330724544 \\
0.0596669035442737	162.5986720434 \\
0.0628072668887092	186.6564862802 \\
};
\addplot [semithick, green!50.0!black, mark=x, mark size=3, mark options={solid}]
table [row sep=\\]{%
0.00314036334443546	4.8950628324 \\
0.00628072668887092	6.1059128792 \\
0.00942109003330638	7.259508666 \\
0.0125614533777418	8.3761219136 \\
0.0157018167221773	9.5430237 \\
0.0188421800666128	10.7143886496 \\
0.0219825434110482	11.825937095 \\
0.0251229067554837	12.9184370644 \\
0.0282632700999191	14.1037889482 \\
0.0314036334443546	15.3157907844 \\
0.0345439967887901	16.4182801374 \\
0.0376843601332255	17.7425859546 \\
0.040824723477661	18.7835732872 \\
0.0439650868220964	19.9347100962 \\
0.0471054501665319	21.1262911152 \\
0.0502458135109674	22.311693099 \\
0.0533861768554028	23.450865014 \\
0.0565265401998383	24.6972985254 \\
0.0596669035442737	25.7617203986 \\
0.0628072668887092	26.7956216014 \\
};
\addplot [semithick, blue, mark=diamond*, mark size=2, mark options={solid}]
table [row sep=\\]{%
0.00314036334443546	2120.597444 \\
0.00628072668887092	1356.57606 \\
0.00942109003330638	1061.612816 \\
0.0125614533777418	877.174761 \\
0.0157018167221773	872.489867 \\
0.0188421800666128	702.231649 \\
0.0219825434110482	767.795181 \\
0.0251229067554837	636.959429 \\
0.0282632700999191	599.070521 \\
0.0314036334443546	584.835232 \\
0.0345439967887901	553.880228 \\
0.0376843601332255	517.542035 \\
0.040824723477661	490.948276 \\
0.0439650868220964	487.489341 \\
0.0471054501665319	462.3048 \\
0.0502458135109674	341.82463 \\
0.0533861768554028	341.549663 \\
0.0565265401998383	338.901639 \\
0.0596669035442737	300.495441 \\
0.0628072668887092	263.027038 \\
};
\addplot [semithick, color0, mark=square*, mark size=2, mark options={solid}]
table [row sep=\\]{%
0.00314036334443546	  456.952\\
0.00628072668887092	  445.818\\
0.00942109003330638	  495.655\\
0.0125614533777418	  438.176\\
0.0157018167221773	  510.129\\
0.0188421800666128	  625.803\\
0.0219825434110482	  569.584\\
0.0251229067554837	  653.585\\
0.0282632700999191	  490.909\\
0.0314036334443546	  563.39 \\
0.0345439967887901	  461.597\\
0.0376843601332255	  664.322\\
0.040824723477661	  580.516\\
0.0439650868220964	  498.659\\
0.0471054501665319	  526.458\\
0.0502458135109674	  549.139\\
0.0533861768554028	  505.442\\
0.0565265401998383	  499.909\\
0.0596669035442737	  557.102 \\
0.0628072668887092	  491.418 \\
};
\end{axis}
\end{tikzpicture}

%% file: paperplots/ol_erdos.tex
\begin{tikzpicture}

\definecolor{color0}{rgb}{0.75,0,0.75}

\begin{axis}[
xlabel={Time (s)},
ylabel={NMSE},
xmin=0, xmax=8,
ymin=0, ymax=1.1,
axis on top,
width=\figW,
height=\figH,
xmajorgrids,
ymajorgrids,
legend entries={{(o)RRMCEX},{(o)ALS},{SGD}},
legend cell align={left}
]
\addplot [green!50.0!black, mark=*, mark size=1, mark options={solid}]
table {%
0.015590456056 0.998548048165781
0.031180912112 0.639185295521262
0.046771368168 0.415308941505856
0.062361824224 0.263262760676984
0.07795228028 0.177251088149082
0.093542736336 0.110115212223275
0.109133192392 0.030471069767632
0.124723648448 0.0154291768784405
0.140314104504 0.00876240176323288
0.15590456056 0.00482026538502937
0.171495016616 0.00282376898043315
0.187085472672 0.00205892670940492
0.202675928728 0.000846646871960679
0.218266384784 0.0006381621053481
0.23385684084 0.000537122631185963
0.249447296896 0.000516424414578266
0.265037752952 0.000467572898312104
0.280628209008 0.000442748377586811
0.296218665064 0.000446868086403265
0.31180912112 0.000439293624554096
0.327399577176 0.000431019938046926
0.342990033232 0.000419735840195845
0.358580489288 0.000422787979514332
0.374170945344 0.000424875920335059
0.3897614014 0.000427615299516994
0.405351857456 0.000441764870615864
0.420942313512 0.000426393650218769
0.436532769568 0.000419862763899147
0.452123225624 0.000428631575231082
0.46771368168 0.000412018001692352
0.483304137736 0.000430953149363689
0.498894593792 0.000443913096663816
0.514485049848 0.000422570553749677
0.530075505904 0.000415446978755061
0.54566596196 0.00042859011836113
0.561256418016 0.000410786686624149
0.576846874072 0.000428153537541377
0.592437330128 0.00042253723537768
0.608027786184 0.000421695852486173
0.62361824224 0.000416520179530587
0.639208698296 0.000430406475162959
0.654799154352 0.000412763347180542
0.670389610408 0.000414236880699536
0.685980066464 0.000422747125876003
0.70157052252 0.000420289048789807
0.717160978576 0.000416491096121529
0.732751434632 0.000411346453967152
0.748341890688 0.000413287772183217
0.763932346744 0.00041263410851987
0.7795228028 0.000414824653259276
0.795113258856 0.000423151474905939
0.810703714912 0.000414489752331809
0.826294170968 0.000411736783802144
0.841884627024 0.000416376132744311
0.85747508308 0.000405233702714602
0.873065539136 0.000416163676055369
0.888655995192 0.000422919026419561
0.904246451248 0.000412624886094887
0.919836907304 0.000408558626823518
0.93542736336 0.000416027807149524
0.951017819416 0.000404370008737678
0.966608275472 0.000414416722428084
0.982198731528 0.000412649915715387
0.997789187584 0.000412283647221634
1.01337964364 0.000409003671467369
1.028970099696 0.000416284420885888
1.044560555752 0.000405531842052286
1.060151011808 0.000406522678706491
1.075741467864 0.000412336288221801
1.09133192392 0.00041118590453519
1.106922379976 0.000408906200986143
1.122512836032 0.000406108496548518
1.138103292088 0.000407204817985758
1.153693748144 0.000405726147784662
1.1692842042 0.000407370903722603
1.184874660256 0.000412567411706336
1.200465116312 0.00040766354612446
1.216055572368 0.000406294223646326
1.231646028424 0.000408965771702922
1.24723648448 0.000401660675089961
1.262826940536 0.000408239167884082
1.278417396592 0.00041198114751971
1.294007852648 0.000406780643316152
1.309598308704 0.000404375156678751
1.32518876476 0.000408764232175813
1.340779220816 0.000401344956610173
1.356369676872 0.000407342926985145
1.371960132928 0.000406677713120649
1.387550588984 0.000406688482575516
1.40314104504 0.000404687098351237
1.418731501096 0.000408774690289274
1.434321957152 0.000402138777688655
1.449912413208 0.00040267015715952
1.465502869264 0.000406444782754293
1.48109332532 0.000406034155343119
1.496683781376 0.000404688458070393
1.512274237432 0.000403047169302346
1.527864693488 0.000403743308140089
1.543455149544 0.000402371967914009
1.5590456056 0.00040351443053691
1.574636061656 0.00040677704935464
1.590226517712 0.000403958044407
1.605816973768 0.000403187739249602
1.621407429824 0.000404804309579287
1.63699788588 0.000400156692545454
1.652588341936 0.00040414404008681
1.668178797992 0.000406310384154014
1.683769254048 0.000403537285450628
1.699359710104 0.000402098983910252
1.71495016616 0.000404716906652566
1.730540622216 0.000400101742108737
1.746131078272 0.000403708327016573
1.761721534328 0.00040340106773044
1.777311990384 0.000403538396080256
1.79290244644 0.000402337672618804
1.808492902496 0.000404712503355236
1.824083358552 0.000400642289398244
1.839673814608 0.000400913768941216
1.855264270664 0.00040328779910109
1.87085472672 0.000403175841927958
1.886445182776 0.000402383475507195
1.902035638832 0.000401391522532345
1.917626094888 0.000401843879168061
1.933216550944 0.000400826865655716
1.948807007 0.000401575954673266
};
\addplot [blue, mark=x, mark size=1, mark options={solid}]
table {%
0.057702836872 1.00253321696559
0.115405673744 1.00021752564742
0.173108510616 1.00017022838009
0.230811347488 0.99975612846202
0.28851418436 0.999239138352756
0.346217021232 0.99871105496811
0.403919858104 0.998467838475993
0.461622694976 0.998110060989822
0.519325531848 0.997644404016507
0.57702836872 0.997121874478759
0.634731205592 0.996573904157985
0.692434042464 0.995951584010332
0.750136879336 0.995277762997348
0.807839716208 0.994442598708198
0.86554255308 0.993650458917369
0.923245389952 0.992767226664488
0.980948226824 0.991807807476085
1.038651063696 0.990723037896433
1.096353900568 0.989429505446009
1.15405673744 0.987803385517993
1.211759574312 0.986276091501296
1.269462411184 0.984556631124129
1.327165248056 0.982603799589893
1.384868084928 0.980440308926796
1.4425709218 0.977797863010373
1.500273758672 0.97454954871245
1.557976595544 0.971508718656441
1.615679432416 0.967927662498326
1.673382269288 0.964023556520046
1.73108510616 0.959757767055007
1.788787943032 0.954530470288053
1.846490779904 0.947987416343747
1.904193616776 0.942052857427292
1.961896453648 0.935011320034281
2.01959929052 0.927116823923536
2.077302127392 0.919050245697286
2.135004964264 0.909347274920394
2.192707801136 0.898738707783042
2.250410638008 0.885281713792904
2.30811347488 0.87235053307481
2.365816311752 0.857912881158517
2.423519148624 0.843066984247342
2.481221985496 0.826643635100145
2.538924822368 0.807497581136742
2.59662765924 0.784728507584912
2.654330496112 0.763559227201149
2.712033332984 0.740775239416437
2.769736169856 0.716191021679808
2.827439006728 0.690540436770688
2.8851418436 0.661377261759641
2.942844680472 0.62930004380947
3.000547517344 0.600480224558852
3.058250354216 0.568695302955028
3.115953191088 0.537138059788386
3.17365602796 0.506127739447194
3.231358864832 0.471817193214032
3.289061701704 0.435657319031414
3.346764538576 0.405398615776419
3.404467375448 0.373397562550773
3.46217021232 0.341600552269325
3.519873049192 0.313539529540934
3.577575886064 0.284637182856195
3.635278722936 0.257706228887436
3.692981559808 0.229774525699221
3.75068439668 0.206453609692577
3.808387233552 0.184081206381644
3.866090070424 0.164782204674462
3.923792907296 0.147491996708134
3.981495744168 0.130156705081779
4.03919858104 0.113816821666687
4.096901417912 0.101230402539048
4.154604254784 0.0896916114703262
4.212307091656 0.0794541758556181
4.270009928528 0.070851287563178
4.3277127654 0.0622835706814965
4.385415602272 0.0548878303301489
4.443118439144 0.0497804774429947
4.500821276016 0.0444150499521134
4.558524112888 0.040189494559413
4.61622694976 0.0370046976877759
4.673929786632 0.0336126564541304
4.731632623504 0.0303550811205645
4.789335460376 0.0287961559288573
4.847038297248 0.0269401278379456
4.90474113412 0.025336777938584
4.962443970992 0.0242272556551233
5.020146807864 0.0231646820737624
5.077849644736 0.0223495990626872
5.135552481608 0.0213531149334968
5.19325531848 0.0208199459140926
5.250958155352 0.0203016525123954
5.308660992224 0.0199511325351736
5.366363829096 0.0197683828154641
5.424066665968 0.019413512751581
5.48176950284 0.0190196827356935
5.539472339712 0.0189035475195647
5.597175176584 0.018781891274746
5.654878013456 0.0186571476979533
5.712580850328 0.0186083581869273
5.7702836872 0.0184481091533535
5.827986524072 0.0183895503190686
5.885689360944 0.0183397362129889
5.943392197816 0.018255566702745
6.001095034688 0.0181929383160294
6.05879787156 0.0182541236320544
6.116500708432 0.0181852308452234
6.174203545304 0.0181080625453669
6.231906382176 0.0181107091981304
6.289609219048 0.0181045381674553
6.34731205592 0.0180884370203077
6.405014892792 0.0180824528750514
6.462717729664 0.0181045196283139
6.520420566536 0.0181256151341835
6.578123403408 0.0180716803014989
6.63582624028 0.0180638675206556
6.693529077152 0.0180515525136378
6.751231914024 0.0180507035686146
6.808934750896 0.0180745998975997
6.866637587768 0.0181034734326811
6.92434042464 0.018059353164167
6.982043261512 0.0180516529637217
7.039746098384 0.0180759926474838
7.097448935256 0.0180587649357303
7.155151772128 0.018070753575213
7.212854609 0.0180564859414512
};
\addplot [color0, mark=+, mark size=1, mark options={solid}]
table {%
0.026077844256 19.3297990381416
0.052155688512 12.296470254466
0.078233532768 7.63908836848487
0.104311377024 5.02582544491778
0.13038922128 3.46650103828986
0.156467065536 2.43964716308796
0.182544909792 1.85311016947212
0.208622754048 1.56829402079574
0.234700598304 1.38520938990605
0.26077844256 1.23159931914616
0.286856286816 1.09204168063586
0.312934131072 0.95458915827007
0.339011975328 0.832481423787363
0.365089819584 0.713145976441507
0.39116766384 0.614037102844886
0.417245508096 0.513245337341881
0.443323352352 0.426605504695022
0.469401196608 0.349102421447606
0.495479040864 0.283339522527257
0.52155688512 0.228722492599973
0.547634729376 0.190704157949921
0.573712573632 0.156006766327807
0.599790417888 0.133877161964708
0.625868262144 0.110791673981021
0.6519461064 0.0920882277171402
0.678023950656 0.0744780467472363
0.704101794912 0.0656915027487706
0.730179639168 0.0567615616704642
0.756257483424 0.050007473545494
0.78233532768 0.0437442919526199
0.808413171936 0.0397207059658242
0.834491016192 0.0319485267783671
0.860568860448 0.0286978360506637
0.886646704704 0.0256194030134697
0.91272454896 0.0229815684694698
0.938802393216 0.0209190203150826
0.964880237472 0.0193145312899599
0.990958081728 0.0167819946803374
1.017035925984 0.0146638860134687
1.04311377024 0.0135333399099405
1.069191614496 0.0123047064162312
1.095269458752 0.0115484140466306
1.121347303008 0.0105904189102809
1.147425147264 0.00969825584374957
1.17350299152 0.00851029092398087
1.199580835776 0.00797080451727496
1.225658680032 0.00735772868695607
1.251736524288 0.00699659688571354
1.277814368544 0.00650726959766562
1.3038922128 0.00611916129889073
1.329970057056 0.00544736033163171
1.356047901312 0.0051767565956993
1.382125745568 0.00487181364277835
1.408203589824 0.00463765589526632
1.43428143408 0.00437610672258569
1.460359278336 0.00428415034108927
1.486437122592 0.0038413509497741
1.512514966848 0.00366484145620169
1.538592811104 0.0034893058162651
1.56467065536 0.00334591452155384
1.590748499616 0.00319977562712842
1.616826343872 0.00315521191775074
1.642904188128 0.00299921429247337
1.668982032384 0.00278653760467188
1.69505987664 0.00268437316701091
1.721137720896 0.00258656184566089
1.747215565152 0.00250534699527853
1.773293409408 0.00242679617104399
1.799371253664 0.00239524795583885
1.82544909792 0.00224979053732622
1.851526942176 0.00218820624290562
1.877604786432 0.00211758923226216
1.903682630688 0.00206924714392631
1.929760474944 0.00200463694781502
1.9558383192 0.00200150768172408
1.981916163456 0.00191069393913053
2.007994007712 0.00186578231232447
2.034071851968 0.00181458663719955
2.060149696224 0.00177990168287152
2.08622754048 0.00172731416469606
2.112305384736 0.00174978623988234
2.138383228992 0.00168965688818114
2.164461073248 0.0016487906386564
2.190538917504 0.00161052396837093
2.21661676176 0.00158536084167074
2.242694606016 0.00154313201203796
2.268772450272 0.00156671823258605
2.294850294528 0.00153472277387145
2.320928138784 0.00149922490459193
2.34700598304 0.00147020245306173
2.373083827296 0.00144754417436358
2.399161671552 0.00141425254579506
2.425239515808 0.0013960530675485
2.451317360064 0.00141415697193913
2.47739520432 0.00138995410641651
2.503473048576 0.00136996775008894
2.529550892832 0.00134697160428893
2.555628737088 0.00132761122423073
2.581706581344 0.0013036158059107
2.6077844256 0.00132453509027599
2.633862269856 0.00131377350773383
2.659940114112 0.00129430791913436
2.686017958368 0.001273669417081
2.712095802624 0.00126107312218715
2.73817364688 0.00123444173994421
2.764251491136 0.00126072169989762
2.790329335392 0.00125759312388284
2.816407179648 0.00123661292173998
2.842485023904 0.00121896910770226
2.86856286816 0.0012094387105312
2.894640712416 0.00118477284304698
2.920718556672 0.0012092492570512
2.946796400928 0.0011949984885925
2.972874245184 0.00119393533093031
2.99895208944 0.00117871049987246
3.025029933696 0.00116836305110116
3.051107777952 0.00114613452927114
3.077185622208 0.00113702832522734
3.103263466464 0.00115674135050088
3.12934131072 0.00115824180986917
3.155419154976 0.00114704210659135
3.181496999232 0.00113407466873974
3.207574843488 0.00112164960318852
3.233652687744 0.00110663166260142
3.259730532 0.00112616957123121
};
\end{axis}

\end{tikzpicture}

%% file: paperplots/ol_temp.tex
\begin{tikzpicture}

\definecolor{color0}{rgb}{0.75,0,0.75}

\begin{axis}[
xlabel={Time (s)},
ylabel={NMSE},
xmin=0, xmax=3.5,
ymin=0, ymax=1.1,
axis on top,
width=\figW,
height=\figH,
xmajorgrids,
ymajorgrids,
legend entries={{(o)RRMCEX},{(o)ALS},{SGD}},
legend cell align={left}
]
\addplot [green!50.0!black, mark=*, mark size=1, mark options={solid}]
table {%
0.00490796105769231 0.999975672275956
0.00981592211538461 0.889492520761911
0.0147238831730769 0.720492167401225
0.0196318442307692 0.449266562065326
0.0245398052884615 0.171542851485626
0.0294477663461538 0.0476519277086306
0.0343557274038462 0.0347896491783817
0.0392636884615385 0.0323203772669232
0.0441716495192308 0.0312347366870316
0.0490796105769231 0.0308784663297781
0.0539875716346154 0.0288069175518712
0.0588955326923077 0.026456531777982
0.06380349375 0.0260449499126988
0.0687114548076923 0.0258821427175
0.0736194158653846 0.0258078430623894
0.0785273769230769 0.0256810685649794
0.0834353379807692 0.0252291219819048
0.0883432990384615 0.0250091044905034
0.0932512600961538 0.0249465440430469
0.0981592211538461 0.0249236682967721
0.103067182211538 0.0249117595907815
0.107975143269231 0.0248308161499968
0.112883104326923 0.0247239778215891
0.117791065384615 0.0246971756088273
0.122699026442308 0.0246865497895403
0.1276069875 0.0246827780404141
0.132514948557692 0.0246716921852989
0.137422909615385 0.0246372810866419
0.142330870673077 0.0246205969877441
0.147238831730769 0.0246156577591427
0.152146792788462 0.0246133594796298
0.157054753846154 0.0246123765283601
0.161962714903846 0.0246050725773878
0.166870675961538 0.0245948432407431
0.171778637019231 0.0245927224640086
0.176686598076923 0.0245917918281616
0.181594559134615 0.0245913543784919
0.186502520192308 0.0245900668120671
0.19141048125 0.0245868867628096
0.196318442307692 0.0245852995606108
0.201226403365385 0.0245848021374603
0.206134364423077 0.0245846192023659
0.211042325480769 0.0245845072644691
0.215950286538462 0.02458371625268
0.220858247596154 0.0245827534021847
0.225766208653846 0.0245825671919396
0.230674169711538 0.0245824759858951
0.235582130769231 0.0245824421116541
0.240490091826923 0.0245822898461965
0.245398052884615 0.0245819676393945
0.250306013942308 0.0245818303755744
0.255213975 0.0245817816662622
};
\addplot [blue, mark=x, mark size=1, mark options={solid}]
table {%
0.0637137527307692 0.999999902197826
0.127427505461538 0.999998096406751
0.191141258192308 0.999739672485998
0.254855010923077 0.877092638816871
0.318568763653846 0.544117246389444
0.382282516384615 0.380879128432801
0.445996269115385 0.306527183209866
0.509710021846154 0.20327166050648
0.573423774576923 0.062835115305778
0.637137527307692 0.0578905877663301
0.700851280038462 0.0581293122082398
0.764565032769231 0.0617066816972242
0.8282787855 0.0606678104391308
0.891992538230769 0.0551143672255638
0.955706290961538 0.0538785419557116
1.01942004369231 0.0518985775944499
1.08313379642308 0.0536788321557763
1.14684754915385 0.0520005633142509
1.21056130188462 0.0539188145554247
1.27427505461538 0.0525884990514139
1.33798880734615 0.0513450646389346
1.40170256007692 0.0511122952901559
1.46541631280769 0.0531726169582854
1.52913006553846 0.0523070271534082
1.59284381826923 0.0517378538171705
1.656557571 0.0511808345760312
1.72027132373077 0.0491302807460466
1.78398507646154 0.0514706026603624
1.84769882919231 0.0486749842174048
1.91141258192308 0.0510512128788085
1.97512633465385 0.0499802256333757
2.03884008738462 0.0491107155364065
2.10255384011538 0.0486285145829427
2.16626759284615 0.0496713530182437
2.22998134557692 0.0485778941528308
2.29369509830769 0.050504054178722
2.35740885103846 0.0481860568898888
2.42112260376923 0.0470651106767992
2.4848363565 0.0498696134404991
2.54855010923077 0.045947330220724
2.61226386196154 0.0472779301039238
2.67597761469231 0.0482805397880083
2.73969136742308 0.0471387704403524
2.80340512015385 0.046325513512222
2.86711887288462 0.0467989699309081
2.93083262561538 0.0461952234928548
2.99454637834615 0.0485414107785436
3.05826013107692 0.0473427618365884
3.12197388380769 0.0448633620726567
3.18568763653846 0.0461685587365122
3.24940138926923 0.0442768845996919
3.313115142 0.0449860634624149
};
\addplot [color0, mark=+, mark size=1, mark options={solid}]
table {%
0.0225986719038462 0.999990629331569
0.0451973438076923 0.999980650264341
0.0677960157115385 0.999966287917756
0.0903946876153846 0.999915628897267
0.112993359519231 0.999813536938502
0.135592031423077 0.999742433162032
0.158190703326923 0.999668270172275
0.180789375230769 0.99955113574439
0.203388047134615 0.999196776249781
0.225986719038462 0.998146337330188
0.248585390942308 0.996749563550029
0.271184062846154 0.995943617879994
0.29378273475 0.994909752608089
0.316381406653846 0.992529946584921
0.338980078557692 0.985124064137647
0.361578750461538 0.97034800427619
0.384177422365385 0.959228931574437
0.406776094269231 0.95106655788659
0.429374766173077 0.937571984570169
0.451973438076923 0.901360561860496
0.474572109980769 0.81820213243109
0.497170781884615 0.732423047342716
0.519769453788462 0.691918911773617
0.542368125692308 0.650668355302099
0.564966797596154 0.567964752729322
0.5875654695 0.415466867037957
0.610164141403846 0.280852244660432
0.632762813307692 0.225824683070263
0.655361485211538 0.198389070770096
0.677960157115385 0.162627283074184
0.700558829019231 0.110890807848994
0.723157500923077 0.0716684680720182
0.745756172826923 0.05600345612996
0.768354844730769 0.0505545204630092
0.790953516634615 0.0464379551037598
0.813552188538462 0.0406874567366583
0.836150860442308 0.0339648990718747
0.858749532346154 0.031127448309264
0.88134820425 0.0291737800131199
0.903946876153846 0.0286271029922099
0.926545548057692 0.0278750499613207
0.949144219961539 0.02662321466013
0.971742891865385 0.0260460038082603
0.994341563769231 0.0254665083255008
1.01694023567308 0.0250690740951322
1.03953890757692 0.0250094737322118
1.06213757948077 0.0248615418567947
1.08473625138462 0.024471039792043
1.10733492328846 0.0245067740943761
1.12993359519231 0.0240684043648063
1.15253226709615 0.0240099135325403
1.175130939 0.0240478680455167
};
\end{axis}

\end{tikzpicture}

%% file: paperplots/online_1.tex
\begin{tikzpicture}

\definecolor{color0}{rgb}{0.75,0,0.75}
\begin{axis}[
height=\figH,
legend cell align={left},
legend entries={{(o)RRMCEX},{  },{(o)ALS},{  },{SGD},{  }},
legend style={draw=white!80.0!black},
tick align=outside,
tick pos=left,
width=\figW,
x grid style={white!69.01960784313725!black},
xlabel={Time (s)},
xmajorgrids,
xmin=0, xmax=220,
y grid style={white!69.01960784313725!black},
ylabel={NMSE},
ymajorgrids,
ymin=-0.0401459074059117, ymax=1.11902279486857,
ytick={-0.2,0,0.2,0.4,0.6,0.8,1,1.2},
yticklabels={−0.2,0.0,0.2,0.4,0.6,0.8,1.0,1.2}
]
\addplot [semithick, green!50.0!black, mark=*, mark size=1, mark options={solid}]
table [row sep=\\]{%
1.45817465742	1.0227218219608 \\
2.91634931484	0.158660057181997 \\
4.37452397226	0.0993659768498004 \\
5.83269862968	0.0835346249301183 \\
7.2908732871	0.0652131594794392 \\
8.74904794452	0.0624132669416407 \\
10.20722260194	0.0517633751160591 \\
11.66539725936	0.0510626980054839 \\
13.12357191678	0.0446876168786 \\
14.5817465742	0.0437290673867245 \\
16.03992123162	0.0401840325320525 \\
17.49809588904	0.0385469975308806 \\
18.95627054646	0.0369329541647187 \\
20.41444520388	0.0346736815396444 \\
21.8726198613	0.0344103829399568 \\
23.33079451872	0.0316639587073505 \\
24.78896917614	0.0323729479302251 \\
26.24714383356	0.0292586462780195 \\
27.70531849098	0.0306862778329865 \\
29.1634931484	0.0272952001615439 \\
30.62166780582	0.0292647127139363 \\
32.07984246324	0.0256654662305394 \\
33.53801712066	0.0280483703287357 \\
34.99619177808	0.0242940497037001 \\
36.4543664355	0.0269932503695939 \\
37.91254109292	0.0231264577726562 \\
39.37071575034	0.0260662076735221 \\
40.82889040776	0.022122185635702 \\
42.28706506518	0.0252419572549709 \\
43.7452397226	0.0212504577881924 \\
45.20341438002	0.0245010689333736 \\
46.66158903744	0.0204874779906549 \\
48.11976369486	0.0238285306999065 \\
49.57793835228	0.0198145815915181 \\
51.0361130097	0.0232126848089966 \\
52.49428766712	0.0192169527288971 \\
53.95246232454	0.0226444284186786 \\
55.41063698196	0.0186827094194785 \\
56.86881163938	0.0221166095714981 \\
58.3269862968	0.0182022364640338 \\
59.78516095422	0.0216235698167289 \\
61.24333561164	0.01776769010774 \\
62.70151026906	0.0211607975254582 \\
64.15968492648	0.0173726246144564 \\
65.6178595839	0.0207246648459983 \\
67.07603424132	0.0170117071250695 \\
68.53420889874	0.0203122278320078 \\
69.99238355616	0.0166804975308126 \\
71.45055821358	0.0199210742631137 \\
72.908732871	0.016375276912303 \\
74.36690752842	0.0195492074746868 \\
75.82508218584	0.0160929127016898 \\
77.28325684326	0.0191949573978836 \\
78.74143150068	0.0158307519077462 \\
80.1996061581	0.0188569121931302 \\
81.65778081552	0.0155865359854673 \\
83.11595547294	0.018533865503917 \\
84.57413013036	0.0153583325375777 \\
86.03230478778	0.0182247755918274 \\
87.4904794452	0.0151444802025055 \\
88.94865410262	0.0179287335381169 \\
90.40682876004	0.0149435439424974 \\
91.86500341746	0.0176449383888609 \\
93.32317807488	0.0147542785849546 \\
94.7813527323	0.0173726776382393 \\
96.23952738972	0.0145755989506542 \\
97.69770204714	0.0171113118321119 \\
99.15587670456	0.0144065552668911 \\
100.61405136198	0.016860262364803 \\
102.0722260194	0.0142463128420439 \\
103.53040067682	0.0166190017606564 \\
104.98857533424	0.0140941351923986 \\
106.44674999166	0.0163870458968182 \\
107.90492464908	0.0139493699781348 \\
109.3630993065	0.0161639477485105 \\
110.82127396392	0.0138114372348355 \\
112.27944862134	0.0159492923328563 \\
113.73762327876	0.0136798194883636 \\
115.19579793618	0.0157426925996418 \\
116.6539725936	0.0135540534209344 \\
118.11214725102	0.0155437860727911 \\
119.57032190844	0.0134337228195527 \\
121.02849656586	0.0153522320889482 \\
122.48667122328	0.0133184525883881 \\
123.9448458807	0.0151677095124927 \\
125.40302053812	0.0132079036469436 \\
126.86119519554	0.0149899148318698 \\
128.31936985296	0.0131017685681988 \\
129.77754451038	0.014818560562024 \\
131.2357191678	0.0129997678369543 \\
132.69389382522	0.0146533738933035 \\
134.15206848264	0.0129016466296675 \\
135.61024314006	0.0144940955394373 \\
137.06841779748	0.0128071720341647 \\
138.5265924549	0.0143404787468177 \\
139.98476711232	0.0127161306415392 \\
141.44294176974	0.0141922884349473 \\
142.90111642716	0.0126283264539184 \\
144.35929108458	0.0140493004439401 \\
145.817465742	0.0125435790611101 \\
};

\addplot [semithick,dotted, green!50.0!black, mark=*, mark size=1, mark options={solid}]
table [row sep=\\]{%
1.37931040526	1.01690367394906 \\
2.75862081052	0.174587557749608 \\
4.13793121578	0.115108089597129 \\
5.51724162104	0.0889207960323206 \\
6.8965520263	0.0743307761906568 \\
8.27586243156	0.0650965347939341 \\
9.65517283682	0.0587781859856644 \\
11.03448324208	0.0542053414251299 \\
12.41379364734	0.0507526553538128 \\
13.7931040526	0.0480564918900619 \\
15.17241445786	0.0458921540620867 \\
16.55172486312	0.0441141593685445 \\
17.93103526838	0.0426246703270164 \\
19.31034567364	0.0413557632605673 \\
20.6896560789	0.0402589885280162 \\
22.06896648416	0.0392989772882079 \\
23.44827688942	0.0384493890609459 \\
24.82758729468	0.0376902626399902 \\
26.20689769994	0.0370062360781287 \\
27.5862081052	0.0363853212709101 \\
28.96551851046	0.0358180423793284 \\
30.34482891572	0.0352968190004583 \\
31.72413932098	0.0348155176791811 \\
33.10344972624	0.0343691214746378 \\
34.4827601315	0.0339534837078966 \\
35.86207053676	0.0335651425951698 \\
37.24138094202	0.0332011804521139 \\
38.62069134728	0.0328591158636057 \\
40.00000175254	0.0325368204521593 \\
41.3793121578	0.0322324541444311 \\
42.75862256306	0.0319444144448946 \\
44.13793296832	0.0316712963837228 \\
45.51724337358	0.0314118606480406 \\
46.89655377884	0.0311650080237747 \\
48.2758641841	0.0309297587324665 \\
49.65517458936	0.0307052355877242 \\
51.03448499462	0.0304906501507405 \\
52.41379539988	0.0302852912558998 \\
53.79310580514	0.030088515422204 \\
55.1724162104	0.0298997387759473 \\
56.55172661566	0.0297184301935205 \\
57.93103702092	0.0295441054369206 \\
59.31034742618	0.0293763221033254 \\
60.68965783144	0.0292146752475802 \\
62.0689682367	0.0290587935653593 \\
63.44827864196	0.0289083360471576 \\
64.82758904722	0.0287629890306788 \\
66.20689945248	0.0286224635927926 \\
67.58620985774	0.0284864932329201 \\
68.965520263	0.0283548318081406 \\
70.34483066826	0.0282272516870223 \\
71.72414107352	0.0281035420945368 \\
73.10345147878	0.0279835076247352 \\
74.48276188404	0.0278669669013641 \\
75.8620722893	0.027753751369459 \\
77.24138269456	0.0276437042033078 \\
78.62069309982	0.0275366793181271 \\
80.00000350508	0.0274325404744286 \\
81.37931391034	0.0273311604654246 \\
82.7586243156	0.0272324203789849 \\
84.13793472086	0.0271362089266558 \\
85.51724512612	0.0270424218330992 \\
86.89655553138	0.0269509612800533 \\
88.27586593664	0.0268617353995528 \\
89.6551763419	0.0267746578117089 \\
91.03448674716	0.0266896472028387 \\
92.41379715242	0.0266066269401651 \\
93.79310755768	0.0265255247196924 \\
95.17241796294	0.0264462722441973 \\
96.5517283682	0.0263688049285815 \\
97.93103877346	0.0262930616300958 \\
99.31034917872	0.0262189844011868 \\
100.68965958398	0.0261465182629348 \\
102.06896998924	0.0260756109972398 \\
103.4482803945	0.026006212956088 \\
104.82759079976	0.0259382768863885 \\
106.20690120502	0.0258717577690059 \\
107.58621161028	0.0258066126707455 \\
108.96552201554	0.0257428006081598 \\
110.3448324208	0.0256802824221482 \\
111.72414282606	0.0256190206624166 \\
113.10345323132	0.0255589794809468 \\
114.48276363658	0.0255001245337027 \\
115.86207404184	0.0254424228898712 \\
117.2413844471	0.025385842947995 \\
118.62069485236	0.0253303543584153 \\
120.00000525762	0.025275927951494 \\
121.37931566288	0.025222535671128 \\
122.75862606814	0.0251701505131174 \\
124.1379364734	0.0251187464679823 \\
125.51724687866	0.025068298467862 \\
126.89655728392	0.0250187823371622 \\
128.27586768918	0.0249701747466422 \\
129.65517809444	0.0249224531706648 \\
131.0344884997	0.0248755958473531 \\
132.41379890496	0.0248295817414207 \\
133.79310931022	0.0247843905094638 \\
135.17241971548	0.0247400024675186 \\
136.55173012074	0.0246963985607085 \\
137.931040526	0.024653560334816 \\
};
\addplot [semithick, blue!50.0!black, mark=x, mark size=1, mark options={solid}]
table [row sep=\\]{%
2.81866009255	1.06633330840155 \\
5.6373201851	1.06597546438215 \\
8.45598027765	0.123867144015933 \\
11.2746403702	0.072332381303455 \\
14.09330046275	0.0619789330986838 \\
16.9119605553	0.0609098259308059 \\
19.73062064785	0.0605988028335093 \\
22.5492807404	0.0589496944351682 \\
25.36794083295	0.0587917395014584 \\
28.1866009255	0.0584701326318897 \\
31.00526101805	0.0570477225643882 \\
33.8239211106	0.0569103647186353 \\
36.64258120315	0.0565620272079088 \\
39.4612412957	0.0565455642564052 \\
42.27990138825	0.0561950813516776 \\
45.0985614808	0.0560239373735696 \\
47.91722157335	0.0558348023611427 \\
50.7358816659	0.055808249571093 \\
53.55454175845	0.0557971249504423 \\
56.373201851	0.0556892412870109 \\
59.19186194355	0.0555830869783605 \\
62.0105220361	0.0554990663345528 \\
64.82918212865	0.055356594072566 \\
67.6478422212	0.0549998865041538 \\
70.46650231375	0.0549909646696336 \\
73.2851624063	0.0549079230204828 \\
76.10382249885	0.0548867505319836 \\
78.9224825914	0.0545979637560372 \\
81.74114268395	0.0545452526245159 \\
84.5598027765	0.0545500357650276 \\
87.37846286905	0.0545376867929511 \\
90.1971229616	0.0545344357206072 \\
93.01578305415	0.0545382809928497 \\
95.8344431467	0.0545025228538877 \\
98.65310323925	0.054493500502313 \\
101.4717633318	0.0544749033439046 \\
104.29042342435	0.0544727453167931 \\
107.1090835169	0.0544501641657269 \\
109.92774360945	0.0544340826649118 \\
112.746403702	0.0544336301429587 \\
115.56506379455	0.0544331956705037 \\
118.3837238871	0.0544221777896282 \\
121.20238397965	0.0544198613349103 \\
124.0210440722	0.0543879859048998 \\
126.83970416475	0.0543839983262208 \\
129.6583642573	0.0543835314858308 \\
132.47702434985	0.0543818665955713 \\
135.2956844424	0.0543775518974775 \\
138.11434453495	0.0543691360534312 \\
140.9330046275	0.0543800858897636 \\
143.75166472005	0.0543773693501686 \\
146.5703248126	0.0543607821887465 \\
149.38898490515	0.0543495681268431 \\
152.2076449977	0.054347914562062 \\
155.02630509025	0.0543340600234519 \\
157.8449651828	0.0543357961973669 \\
160.66362527535	0.0543211420695511 \\
163.4822853679	0.0543147298897588 \\
166.30094546045	0.0543064275868318 \\
169.119605553	0.0542999266764243 \\
171.93826564555	0.0542961734924315 \\
174.7569257381	0.0542914652933924 \\
177.57558583065	0.0542927845881424 \\
180.3942459232	0.0542878788410109 \\
183.21290601575	0.0542870247285608 \\
186.0315661083	0.0542675301237939 \\
188.85022620085	0.0542695201101121 \\
191.6688862934	0.0542676875908629 \\
194.48754638595	0.0542463894553497 \\
197.3062064785	0.0542399171061893 \\
200.12486657105	0.0542369893494531 \\
202.9435266636	0.054226644927876 \\
205.76218675615	0.0542317217317021 \\
208.5808468487	0.0542127604403625 \\
211.39950694125	0.0542221906783084 \\
214.2181670338	0.0542164211739715 \\
217.03682712635	0.0542127261921406 \\
219.8554872189	0.0542063395270813 \\
222.67414731145	0.0542093311025014 \\
225.492807404	0.0542042251060082 \\
228.31146749655	0.0542087689817003 \\
231.1301275891	0.0542136484851067 \\
233.94878768165	0.0542067095871546 \\
236.7674477742	0.0542065124062835 \\
239.58610786675	0.0542110742871847 \\
242.4047679593	0.0541962191703606 \\
245.22342805185	0.0541941322460156 \\
248.0420881444	0.0541881850190659 \\
250.86074823695	0.0541873461052969 \\
253.6794083295	0.0541794911542324 \\
256.49806842205	0.054182448568581 \\
259.3167285146	0.0541812813229126 \\
262.13538860715	0.0541813995794354 \\
264.9540486997	0.0541792903613669 \\
267.77270879225	0.0541725174165772 \\
270.5913688848	0.0541711638444504 \\
273.41002897735	0.0541662534739205 \\
276.2286890699	0.0541652778303557 \\
279.04734916245	0.0541718544665698 \\
281.866009255	0.054172407559944 \\
};

\addplot [semithick,dotted, blue!50.0!black, mark=x, mark size=1, mark options={solid}]
table [row sep=\\]{%
2.88268073427	1.02074828094897 \\
5.76536146854	0.989896037560213 \\
8.64804220281	0.51143932183887 \\
11.53072293708	0.470241660214355 \\
14.41340367135	0.461836589729593 \\
17.29608440562	0.38047062532363 \\
20.17876513989	0.370641657718752 \\
23.06144587416	0.367451339618649 \\
25.94412660843	0.313856829929659 \\
28.8268073427	0.313356671086849 \\
31.70948807697	0.311649256703823 \\
34.59216881124	0.311173231998625 \\
37.47484954551	0.31115160607583 \\
40.35753027978	0.310470818764557 \\
43.24021101405	0.310420881299425 \\
46.12289174832	0.310428436855557 \\
49.00557248259	0.310431037346879 \\
51.88825321686	0.310400693055276 \\
54.77093395113	0.310403269331063 \\
57.6536146854	0.310359686877641 \\
60.53629541967	0.310357132975113 \\
63.41897615394	0.310292942964471 \\
66.30165688821	0.310272929334037 \\
69.18433762248	0.310288319417782 \\
72.06701835675	0.310241098037418 \\
74.94969909102	0.310252557733705 \\
77.83237982529	0.310220369627408 \\
80.71506055956	0.31021487798674 \\
83.59774129383	0.310207411258844 \\
86.4804220281	0.310197724866668 \\
89.36310276237	0.310210980664211 \\
92.24578349664	0.310224020157845 \\
95.12846423091	0.310205774378631 \\
98.01114496518	0.310184189261041 \\
100.89382569945	0.310202492771859 \\
103.77650643372	0.31019666933907 \\
106.65918716799	0.310215669696993 \\
109.54186790226	0.310195591260731 \\
112.42454863653	0.310205136371616 \\
115.3072293708	0.310199905014578 \\
118.18991010507	0.310196070244072 \\
121.07259083934	0.310196729038716 \\
123.95527157361	0.310199077670843 \\
126.83795230788	0.310176364185524 \\
129.72063304215	0.310165390972179 \\
132.60331377642	0.310201338489471 \\
135.48599451069	0.310193763902483 \\
138.36867524496	0.310175133392445 \\
141.25135597923	0.310193274914692 \\
144.1340367135	0.310169351319469 \\
147.01671744777	0.31017568911637 \\
149.89939818204	0.310169992908924 \\
152.78207891631	0.310193558451154 \\
155.66475965058	0.310194299031404 \\
158.54744038485	0.310178610838213 \\
161.43012111912	0.310194840840672 \\
164.31280185339	0.31019510328492 \\
167.19548258766	0.310210232542928 \\
170.07816332193	0.310194148155381 \\
172.9608440562	0.310217258873847 \\
175.84352479047	0.310192278032047 \\
178.72620552474	0.310198998795855 \\
181.60888625901	0.310194205809484 \\
184.49156699328	0.310188367770915 \\
187.37424772755	0.310191278157825 \\
190.25692846182	0.310207750164831 \\
193.13960919609	0.31019552414642 \\
196.02228993036	0.310194763503676 \\
198.90497066463	0.310192577028395 \\
201.7876513989	0.310157176431351 \\
204.67033213317	0.310173289230209 \\
207.55301286744	0.310184666797477 \\
210.43569360171	0.310192181168292 \\
213.31837433598	0.310221967486966 \\
216.20105507025	0.310185144491175 \\
219.08373580452	0.310180838890232 \\
221.96641653879	0.31020666431886 \\
224.84909727306	0.310194750583968 \\
227.73177800733	0.310209237897605 \\
230.6144587416	0.310188761023718 \\
233.49713947587	0.310200582690956 \\
236.37982021014	0.310185195301528 \\
239.26250094441	0.310156747166517 \\
242.14518167868	0.310165295236553 \\
245.02786241295	0.310169030318414 \\
247.91054314722	0.310180319295326 \\
250.79322388149	0.310247092227489 \\
253.67590461576	0.310232819669205 \\
256.55858535003	0.310235567850774 \\
259.4412660843	0.31020708760115 \\
262.32394681857	0.310227247746114 \\
265.20662755284	0.310186316535227 \\
268.08930828711	0.310230928387436 \\
270.97198902138	0.310228083846613 \\
273.85466975565	0.310210187528808 \\
276.73735048992	0.310216772908369 \\
279.62003122419	0.310212869920329 \\
282.50271195846	0.310179382903198 \\
285.38539269273	0.310216844669365 \\
288.268073427	0.310205145500685 \\
};

\addplot [semithick, color0, mark=+, mark size=1, mark options={solid}]
table [row sep=\\]{%
2.74341947933	1.00040869827022 \\
5.48683895866	1.00021986040032 \\
8.23025843799	1.00017450465871 \\
10.97367791732	1.00010785941633 \\
13.71709739665	1.00007108271168 \\
16.46051687598	1.00002496330087 \\
19.20393635531	1.00001551618289 \\
21.94735583464	0.999993172572296 \\
24.69077531397	0.999890577908086 \\
27.4341947933	0.999662967877608 \\
30.17761427263	0.999635324815844 \\
32.92103375196	0.999538122900137 \\
35.66445323129	0.99917862498506 \\
38.40787271062	0.999137865755787 \\
41.15129218995	0.998826530767039 \\
43.89471166928	0.998715896636032 \\
46.63813114861	0.998666339527798 \\
49.38155062794	0.998351762006479 \\
52.12497010727	0.997712921594169 \\
54.8683895866	0.997524708060791 \\
57.61180906593	0.996830919884674 \\
60.35522854526	0.994106214225847 \\
63.09864802459	0.993399374808534 \\
65.84206750392	0.992157687795542 \\
68.58548698325	0.987081921008829 \\
71.32890646258	0.986555890843092 \\
74.07232594191	0.982400784057347 \\
76.81574542124	0.976920075935881 \\
79.55916490057	0.976002889933332 \\
82.3025843799	0.974509621502635 \\
85.04600385923	0.964986690360264 \\
87.78942333856	0.962044769890275 \\
90.53284281789	0.949332317538789 \\
93.27626229722	0.936307277507239 \\
96.01968177655	0.909276857936936 \\
98.76310125588	0.889194541308568 \\
101.50652073521	0.853075048182447 \\
104.24994021454	0.847740049968376 \\
106.99335969387	0.827956816487584 \\
109.7367791732	0.815757864736893 \\
112.48019865253	0.801382917695748 \\
115.22361813186	0.774957252674353 \\
117.96703761119	0.770885896940953 \\
120.71045709052	0.746837906676323 \\
123.45387656985	0.680743629012151 \\
126.19729604918	0.648085188414358 \\
128.94071552851	0.604238430632023 \\
131.68413500784	0.459022550996021 \\
134.42755448717	0.41422568432888 \\
137.1709739665	0.364835346679069 \\
139.91439344583	0.317788256124887 \\
142.65781292516	0.274985884953003 \\
145.40123240449	0.262643496375523 \\
148.14465188382	0.249877208275273 \\
150.88807136315	0.218421095563724 \\
153.63149084248	0.210972175513693 \\
156.37491032181	0.208974235355182 \\
159.11832980114	0.207258419147711 \\
161.86174928047	0.205438740846922 \\
164.6051687598	0.180367334566773 \\
167.34858823913	0.176249216785481 \\
170.09200771846	0.163796503024582 \\
172.83542719779	0.15233278141525 \\
175.57884667712	0.148264354773876 \\
178.32226615645	0.0996814798678439 \\
181.06568563578	0.0932366977169719 \\
183.80910511511	0.0843061808212318 \\
186.55252459444	0.0813320355289961 \\
189.29594407377	0.0768671233636335 \\
192.0393635531	0.0706927635166896 \\
194.78278303243	0.0701743893909032 \\
197.52620251176	0.069176014100252 \\
200.26962199109	0.0652321487544418 \\
203.01304147042	0.0651439448300147 \\
205.75646094975	0.0617473821320161 \\
208.49988042908	0.0612521021692014 \\
211.24329990841	0.0604842728655528 \\
213.98671938774	0.0601085203737737 \\
216.73013886707	0.0590155747295631 \\
219.4735583464	0.0588365234099402 \\
222.21697782573	0.058425449098947 \\
224.96039730506	0.0571932023478015 \\
227.70381678439	0.057187577652567 \\
230.44723626372	0.0571691328845471 \\
233.19065574305	0.0570280866478221 \\
235.93407522238	0.056841590641062 \\
238.67749470171	0.0566542198838985 \\
241.42091418104	0.0560655249973312 \\
244.16433366037	0.0560099256771654 \\
246.9077531397	0.0559911649565826 \\
249.65117261903	0.0559496024836444 \\
252.39459209836	0.0556655729919851 \\
255.13801157769	0.0555048302350111 \\
257.88143105702	0.0553467617962634 \\
260.62485053635	0.0552693710243178 \\
263.36827001568	0.055013121154356 \\
266.11168949501	0.0550130965759022 \\
268.85510897434	0.0550058697805704 \\
271.59852845367	0.0549977889787377 \\
274.341947933	0.0549909099486663 \\
};
\addplot [semithick,dotted, color0, mark=+, mark size=1, mark options={solid}]
table [row sep=\\]{%
2.76739584189	1.00040211453466 \\
5.53479168378	1.00020848930421 \\
8.30218752567	1.00016024783171 \\
11.06958336756	1.00010626118083 \\
13.83697920945	1.00009825986468 \\
16.60437505134	1.00007309096703 \\
19.37177089323	1.00003518632361 \\
22.13916673512	0.999957858899055 \\
24.90656257701	0.999945846670708 \\
27.6739584189	0.999829582273512 \\
30.44135426079	0.999772646244949 \\
33.20875010268	0.9994720320964 \\
35.97614594457	0.999293599587706 \\
38.74354178646	0.998924909538766 \\
41.51093762835	0.997265556384661 \\
44.27833347024	0.997017362104008 \\
47.04572931213	0.995911264021124 \\
49.81312515402	0.994605380177737 \\
52.58052099591	0.993485029367379 \\
55.3479168378	0.990938856379805 \\
58.11531267969	0.990719677558426 \\
60.88270852158	0.984431007715799 \\
63.65010436347	0.978143191218441 \\
66.41750020536	0.964557695349972 \\
69.18489604725	0.949129851082674 \\
71.95229188914	0.892287971372865 \\
74.71968773103	0.853016907060775 \\
77.48708357292	0.845165037146708 \\
80.25447941481	0.838893709575731 \\
83.0218752567	0.789810145819394 \\
85.78927109859	0.766742332123598 \\
88.55666694048	0.765160437939647 \\
91.32406278237	0.624179390818007 \\
94.09145862426	0.616979982145477 \\
96.85885446615	0.606271144150671 \\
99.62625030804	0.594084496874073 \\
102.39364614993	0.579499578843535 \\
105.16104199182	0.576812988137071 \\
107.92843783371	0.56766298698896 \\
110.6958336756	0.550275360956935 \\
113.46322951749	0.547602963535384 \\
116.23062535938	0.534900708489277 \\
118.99802120127	0.523784638654627 \\
121.76541704316	0.50405178135396 \\
124.53281288505	0.478194873910669 \\
127.30020872694	0.458546539275045 \\
130.06760456883	0.453366694524209 \\
132.83500041072	0.431903093452957 \\
135.60239625261	0.414748522768386 \\
138.3697920945	0.412444788501311 \\
141.13718793639	0.387110201295933 \\
143.90458377828	0.385560981900152 \\
146.67197962017	0.384910241035758 \\
149.43937546206	0.382815784249634 \\
152.20677130395	0.381762098243779 \\
154.97416714584	0.379344395248783 \\
157.74156298773	0.378970411667896 \\
160.50895882962	0.374246728954517 \\
163.27635467151	0.371762739126401 \\
166.0437505134	0.366737196403513 \\
168.81114635529	0.353772064410676 \\
171.57854219718	0.343897211426146 \\
174.34593803907	0.343787682860851 \\
177.11333388096	0.343589302316151 \\
179.88072972285	0.341539761934601 \\
182.64812556474	0.336123683932507 \\
185.41552140663	0.336028332642722 \\
188.18291724852	0.33536976583826 \\
190.95031309041	0.335293919441155 \\
193.7177089323	0.334394280018282 \\
196.48510477419	0.334297500028684 \\
199.25250061608	0.33368634458661 \\
202.01989645797	0.333472017751011 \\
204.78729229986	0.332890688472543 \\
207.55468814175	0.332636636078465 \\
210.32208398364	0.331704654518454 \\
213.08947982553	0.331447940320223 \\
215.85687566742	0.331372928942533 \\
218.62427150931	0.330884449924415 \\
221.3916673512	0.33085711597957 \\
224.15906319309	0.330804404307215 \\
226.92645903498	0.330623093453403 \\
229.69385487687	0.330268362996796 \\
232.46125071876	0.330084975021724 \\
235.22864656065	0.330000025344841 \\
237.99604240254	0.329963853557464 \\
240.76343824443	0.329941238378925 \\
243.53083408632	0.329836213730358 \\
246.29822992821	0.329723646416898 \\
249.0656257701	0.329677748797961 \\
251.83302161199	0.329498266321772 \\
254.60041745388	0.32948058878053 \\
257.36781329577	0.329412279757071 \\
260.13520913766	0.329411959174713 \\
262.90260497955	0.329404455456922 \\
265.67000082144	0.329387481592075 \\
268.43739666333	0.329384313629548 \\
271.20479250522	0.329361754085243 \\
273.97218834711	0.329345884726438 \\
276.739584189	0.32928514587668 \\
};

\end{axis}

\end{tikzpicture}

%% file: bibliography.bbl

%% file: journal_kernel.bbl
\begin{thebibliography}{10}
\providecommand{\url}[1]{#1}
\csname url@samestyle\endcsname
\providecommand{\newblock}{\relax}
\providecommand{\bibinfo}[2]{#2}
\providecommand{\BIBentrySTDinterwordspacing}{\spaceskip=0pt\relax}
\providecommand{\BIBentryALTinterwordstretchfactor}{4}
\providecommand{\BIBentryALTinterwordspacing}{\spaceskip=\fontdimen2\font plus
\BIBentryALTinterwordstretchfactor\fontdimen3\font minus
  \fontdimen4\font\relax}
\providecommand{\BIBforeignlanguage}[2]{{%
\expandafter\ifx\csname l@#1\endcsname\relax
\typeout{** WARNING: IEEEtran.bst: No hyphenation pattern has been}%
\typeout{** loaded for the language `#1'. Using the pattern for}%
\typeout{** the default language instead.}%
\else
\language=\csname l@#1\endcsname
\fi
#2}}
\providecommand{\BIBdecl}{\relax}
\BIBdecl

\bibitem{candes}
E.~J. Cand{\`e}s and B.~Recht, ``{Exact matrix completion via convex
  optimization},'' \emph{Foundations of Computational Mathematics}, vol.~9,
  no.~6, pp. 717--772, Dec. 2009.

\bibitem{ji2010robust}
H.~Ji, C.~Liu, Z.~Shen, and Y.~Xu, ``{Robust video denoising using low rank
  matrix completion},'' in \emph{Proc. of Computer Vision and Pattern
  Recognition Conf.}, San Francisco, USA, Jun. 2010, pp. 1791--1798.

\bibitem{yi2015partial}
K.~Yi, J.~Wan, T.~Bao, and L.~Yao, ``{A DCT regularized matrix completion
  algorithm for energy efficient data gathering in wireless sensor networks},''
  \emph{Int. Journal of Distributed Sensor Networks}, vol.~11, no.~7, p.
  272761, Jul. 2015.

\bibitem{koren2009matrix}
Y.~Koren, R.~Bell, and C.~Volinsky, ``{Matrix factorization techniques for
  recommender systems},'' \emph{Computer}, vol.~42, no.~8, pp. 30--37, Aug.
  2009.

\bibitem{candes2010matrix}
E.~J. Candes and Y.~Plan, ``{Matrix completion with noise},'' \emph{Proceedings
  of the IEEE}, vol.~98, no.~6, pp. 925--936, Jun. 2010.

\bibitem{cheng2013stcdg}
J.~Cheng, Q.~Ye, H.~Jiang, D.~Wang, and C.~Wang, ``{STCDG: an efficient data
  gathering algorithm based on matrix completion for wireless sensor
  networks},'' \emph{IEEE Transactions on Wireless Communications}, vol.~12,
  no.~2, pp. 850--861, Feb. 2013.

\bibitem{kalofolias2014matrix}
V.~Kalofolias, X.~Bresson, M.~Bronstein, and P.~Vandergheynst, ``Matrix
  completion on graphs,'' in \emph{{Neural Information Processing Systems
  Workshop ``Out of the Box: Robustness in High Dimension"}}, Montreal, Canada,
  Dec. 2014.

\bibitem{chen}
S.~Chen, A.~Sandryhaila, J.~M. Moura, and J.~Kovacevi{\'c}, ``Signal recovery
  on graphs: Variation minimization,'' \emph{IEEE Transactions on Signal
  Processing}, vol.~63, no.~17, pp. 4609--4624, Sep. 2015.

\bibitem{rao2015collaborative}
N.~Rao, H.-F. Yu, P.~K. Ravikumar, and I.~S. Dhillon, ``Collaborative filtering
  with graph information: Consistency and scalable methods,'' in
  \emph{{Advances in Neural Information Processing Systems}}, Montreal, Canada,
  Dec. 2015, pp. 2107--2115.

\bibitem{ma2011recommender}
H.~Ma, D.~Zhou, C.~Liu, M.~R. Lyu, and I.~King, ``{Recommender systems with
  social regularization},'' in \emph{{Proc. of ACM Int. Conf. on Web Search and
  Data Mining}}, Hong Kong, Feb. 2011, pp. 287--296.

\bibitem{abernethy2006low}
J.~Abernethy, F.~Bach, T.~Evgeniou, and J.-P. Vert, ``{Low-rank matrix
  factorization with attributes},'' Ecole des mines de Paris, Tech. Rep., Sept.
  2006.

\bibitem{bazerque2013}
J.~A. Bazerque and G.~B. Giannakis, ``{Nonparametric basis pursuit via sparse
  kernel-based learning: A unifying view with advances in blind methods},''
  \emph{IEEE Signal Processing Magazine}, vol.~30, no.~4, pp. 112--125, Jul.
  2013.

\bibitem{zhou2012}
T.~Zhou, H.~Shan, A.~Banerjee, and G.~Sapiro, ``{Kernelized probabilistic
  matrix factorization: Exploiting graphs and side information},'' in
  \emph{{Proc. of SIAM Int. Conf. on Data Mining}}, Minneapolis, USA, Jul.
  2012, pp. 403--414.

\bibitem{stock2018comparative}
M.~Stock, T.~Pahikkala, A.~Airola, B.~De~Baets, and W.~Waegeman, ``A
  comparative study of pairwise learning methods based on kernel ridge
  regression,'' \emph{{Neural Computation}}, vol.~30, no.~8, pp. 2245--2283,
  2018.

\bibitem{ma2011fixed}
S.~Ma, D.~Goldfarb, and L.~Chen, ``{Fixed point and Bregman iterative methods
  for matrix rank minimization},'' \emph{Mathematical Programming}, vol. 128,
  no. 1-2, pp. 321--353, Jun. 2011.

\bibitem{cai2010singular}
J.~F. Cai, E.~J. Cand{\`e}s, and Z.~Shen, ``{A singular value thresholding
  algorithm for matrix completion},'' \emph{SIAM Journal on Optimization},
  vol.~20, no.~4, pp. 1956--1982, Jan. 2010.

\bibitem{gimenez}
P.~Gim\'{e}nez-Febrer and A.~Pag\`{e}s-Zamora, ``Matrix completion of noisy
  graph signals via proximal gradient minimization,'' in \emph{Proc. of IEEE
  Int. Conf. on Acoustics, Speech and Signal Processing}, New Orleans, USA,
  March 2017, pp. 4441--4445.

\bibitem{hastie2015matrix}
T.~Hastie, R.~Mazumder, J.~D. Lee, and R.~Zadeh, ``{Matrix completion and
  low-rank SVD via fast alternating least squares},'' \emph{Journal of Machine
  Learning Research}, vol.~16, pp. 3367--3402, Jan. 2015.

\bibitem{jain2013low}
P.~Jain, P.~Netrapalli, and S.~Sanghavi, ``{Low-rank matrix completion using
  alternating minimization},'' in \emph{{Proc. of ACM Symp. on Theory of
  Computing}}, Palo Alto, USA, Jun. 2013, pp. 665--674.

\bibitem{gemulla2011large}
R.~Gemulla, E.~Nijkamp, P.~J. Haas, and Y.~Sismanis, ``{Large-scale matrix
  factorization with distributed stochastic gradient descent},'' in
  \emph{{Proc. of ACM SIGKDD Int. Conf. on Knowledge Discovery and Data
  Mining}}, San Diego, USA, Aug. 2011, pp. 69--77.

\bibitem{shawe}
J.~Shawe-Taylor and N.~Cristianini, \emph{{Kernel Methods for Pattern
  Analysis}}.\hskip 1em plus 0.5em minus 0.4em\relax Cambridge University
  Press, 2004.

\bibitem{friedman2001elements}
J.~Friedman, T.~Hastie, and R.~Tibshirani, \emph{{The Elements of Statistical
  Learning}}.\hskip 1em plus 0.5em minus 0.4em\relax Springer Series in
  Statistics, 2001.

\bibitem{romero}
D.~Romero, M.~Ma, and G.~B. Giannakis, ``{Kernel-based reconstruction of graph
  signals},'' \emph{IEEE Transactions on Signal Processing}, vol.~65, no.~3,
  pp. 764--778, Feb. 2017.

\bibitem{srebro2005maximum}
N.~Srebro, J.~Rennie, and T.~S. Jaakkola, ``{Maximum-margin matrix
  factorization},'' in \emph{{Advances in Neural Information Processing
  Systems}}, Vancouver, Canada, Dec. 2005, pp. 1329--1336.

\bibitem{teflioudi2012}
C.~Teflioudi, F.~Makari, and R.~Gemulla, ``{Distributed matrix completion},''
  in \emph{{Proc. of Int. Conf. on Data Mining}}, Brussels, Belgium, Dec. 2012,
  pp. 655--664.

\bibitem{mardani2015}
M.~Mardani, G.~Mateos, and G.~B. Giannakis, ``Subspace learning and imputation
  for streaming big data matrices and tensors,'' \emph{IEEE Transactions on
  Signal Processing}, vol.~63, no.~10, pp. 2663--2677, May 2015.

\bibitem{scholkopf2001generalized}
B.~Sch{\"o}lkopf, R.~Herbrich, and A.~J. Smola, ``A generalized representer
  theorem,'' in \emph{International conference on computational learning
  theory}.\hskip 1em plus 0.5em minus 0.4em\relax Springer, 2001, pp. 416--426.

\bibitem{henderson1981}
H.~V. Henderson and S.~R. Searle, ``On deriving the inverse of a sum of
  matrices,'' \emph{Siam Review}, vol.~23, no.~1, pp. 53--60, 1981.

\bibitem{pahikkala2014two}
T.~Pahikkala, M.~Stock, A.~Airola, T.~Aittokallio, B.~De~Baets, and
  W.~Waegeman, ``A two-step learning approach for solving full and almost full
  cold start problems in dyadic prediction,'' in \emph{{Joint European Conf. on
  Machine Learning and Knowledge Discovery in Databases}}.\hskip 1em plus 0.5em
  minus 0.4em\relax Springer, 2014, pp. 517--532.

\bibitem{drineas2005nystrom}
P.~Drineas and M.~W. Mahoney, ``{On the Nystr{\"o}m method for approximating a
  Gram matrix for improved kernel-based learning},'' \emph{Journal of Machine
  Learning Research}, vol.~6, pp. 2153--2175, Dec. 2005.

\bibitem{alaoui2015fast}
A.~Alaoui and M.~W. Mahoney, ``{Fast randomized kernel ridge regression with
  statistical guarantees},'' in \emph{{Advances in Neural Information
  Processing Systems}}, Montreal, Canada, Dec. 2015, pp. 775--783.

\bibitem{yang2015randomized}
Y.~Yang, M.~Pilanci, and M.~J. Wainwright, ``{Randomized sketches for kernels:
  Fast and optimal non-parametric regression},'' \emph{The Annals of
  Statistics}, vol.~45, no.~3, pp. 991--1023, Jun. 2017.

\bibitem{avron2017faster}
H.~Avron, K.~L. Clarkson, and D.~P. Woodruff, ``{Faster kernel ridge regression
  using sketching and preconditioning},'' \emph{SIAM Journal on Matrix Analysis
  and Applications}, vol.~38, no.~4, pp. 1116--1138, Jan. 2017.

\bibitem{van2014online}
S.~Van~Vaerenbergh and I.~Santamar{\'\i}a, ``{Online Regression with
  Kernels},'' in \emph{Regularization, Optimization, Kernels, and Support
  Vector Machines}.\hskip 1em plus 0.5em minus 0.4em\relax {Chapman and
  Hall/CRC}, 2014, ch.~21, pp. 477--501.

\bibitem{lu2016large}
J.~Lu, S.~C. Hoi, J.~Wang, P.~Zhao, and Z.-Y. Liu, ``{Large scale online kernel
  learning},'' \emph{Journal of Machine Learning Research}, vol.~17, no.~47,
  pp. 1--43, Jan. 2016.

\bibitem{sheikholeslami2018}
F.~Sheikholeslami, D.~Berberidis, and G.~B. Giannakis, ``Large-scale
  kernel-based feature extraction via low-rank subspace tracking on a budget,''
  \emph{IEEE Transactions on Signal Processing}, vol.~66, no.~8, pp.
  1967--1981, April 2018.

\bibitem{bottou2012stochastic}
L.~Bottou, ``{Stochastic gradient descent tricks},'' in \emph{{Neural networks:
  Tricks of the Trade}}.\hskip 1em plus 0.5em minus 0.4em\relax Springer, 2012,
  pp. 421--436.

\bibitem{schizas}
I.~D. Schizas, G.~Mateos, and G.~B. Giannakis, ``Distributed lms for
  consensus-based in-network adaptive processing,'' \emph{IEEE Transactions on
  Signal Processing}, vol.~57, no.~6, pp. 2365--2382, June 2009.

\bibitem{shuman2013emerging}
D.~I. Shuman, S.~K. Narang, P.~Frossard, A.~Ortega, and P.~Vandergheynst,
  ``{The emerging field of signal processing on graphs: Extending
  high-dimensional data analysis to networks and other irregular domains},''
  \emph{IEEE Signal Processing Magazine}, vol.~30, no.~3, pp. 83--98, May 2013.

\bibitem{smola2003kernels}
A.~J. Smola and R.~Kondor, ``{Kernels and regularization on graphs},'' in
  \emph{{Learning Theory and Kernel Machines}}.\hskip 1em plus 0.5em minus
  0.4em\relax Springer, 2003, pp. 144--158.

\bibitem{romerospace}
D.~Romero, V.~N. Ioannidis, and G.~B. Giannakis, ``Kernel-based reconstruction
  of space-time functions on dynamic graphs,'' \emph{IEEE Journal of Selected
  Topics in Signal Processing}, vol.~11, no.~6, pp. 856--869, Sept. 2017.

\bibitem{cotter2011explicit}
A.~Cotter, J.~Keshet, and N.~Srebro, ``Explicit approximations of the gaussian
  kernel,'' \emph{arXiv preprint arXiv:1109.4603}, 2011.

\bibitem{rahimi2008random}
A.~Rahimi and B.~Recht, ``Random features for large-scale kernel machines,'' in
  \emph{{Advances in neural Information Processing Systems}}, 2008, pp.
  1177--1184.

\bibitem{lax}
P.~Lax, \emph{{Linear Algebra and its Applications}}.\hskip 1em plus 0.5em
  minus 0.4em\relax Wiley, 2007.

\end{thebibliography}
